\documentclass{article}
\usepackage[pdfencoding=auto, psdextra,colorlinks,citecolor=blue]{hyperref}
\PassOptionsToPackage{numbers,compress,sort}{natbib}

\usepackage{bbold}

\usepackage{relsize}
\usepackage{comment}

\usepackage{enumitem}

\usepackage[table]{xcolor}
\usepackage{microtype}
\usepackage{graphicx}
\usepackage{booktabs} %
\usepackage{amsmath}
\usepackage{amssymb}
\usepackage[colorinlistoftodos,textsize=scriptsize]{todonotes}
\usepackage{marginnote}
\usepackage{amsthm}
\usepackage{booktabs}
\usepackage{multirow}
\usepackage{float}
\usepackage{subcaption}
\usepackage{wrapfig}
\graphicspath{{./Figures/}}

\usepackage[normalem]{ulem}
\useunder{\uline}{\ul}{}

\usepackage{amsmath,amsfonts,bm}

\usepackage{thmtools}
\usepackage{thm-restate}

\newtheorem{theorem}{Theorem}
\newtheorem{corollary}{Corollary}
\newtheorem{lemma}{Lemma}

\newtheorem{definition}{Definition}

\newcommand{\prob}{\mathbb{P}}
\newcommand{\acc}{\mathbf{a}}
\newcommand{\complex}{\mathcal{C}}
\newcommand{\mislabeled}{\mathcal{M}}
\newcommand{\atypical}{\mathcal{R}}
\newcommand{\risk}{\mathcal{L}}
\newcommand{\learn}{\mathbf{fslt}}
\newcommand{\forget}{\mathbf{ssft}}
\newcommand{\dataset}{\mathcal{S}}
\newcommand{\dist}{\mathcal{D}}
\newcommand{\distX}{\mathcal{X}}
\newcommand{\distY}{\mathcal{Y}}
\newcommand{\f}{f}

\newcommand{\x}{\mathbf{x}}
\newcommand{\y}{\mathbf{y}}
\newcommand{\z}{\mathbf{z}}

\newcommand{\W}{\mathbf{w}}
\newcommand{\transpose}{\top}
\newcommand{\loss}{\ell}

\newcommand{\identity}{\mathbf{I}}
\newcommand{\direction}{\mathbf{u}}
\newcommand{\indices}{\mathcal{I}}

\newcommand{\idx}[2]{[#1]_{#2}}

\def\eqref#1{equation~\ref{#1}}

\def\1{\bm{1}}

\def\vmu{{\bm{\mu}}}

\DeclareMathAlphabet{\mathsfit}{\encodingdefault}{\sfdefault}{m}{sl}
\SetMathAlphabet{\mathsfit}{bold}{\encodingdefault}{\sfdefault}{bx}{n}

\def\gX{{\mathcal{X}}}
\def\gY{{\mathcal{Y}}}

\DeclareMathOperator*{\argmin}{argmin\;}

\DeclareMathOperator{\sign}{sign}

\newcommand\norm[2]{\left|\!\left|#1 \right|\!\right|_{#2}}

\newcommand{\abs}[1]{\left\lvert#1 \right\rvert}

\newif\ifsubmit
\submitfalse
\ifsubmit
\newcommand{\pratyush}[1]{}
\newcommand{\update}[1]{}
\else
\newcommand{\update}[1]{\textcolor{black}{#1}}
\newcommand{\pratyush}[1]{\textcolor{orange}{[PM: #1]}}

\fi
\definecolor{darkgreen}{rgb}{0,0.3,0}
\definecolor{darkblue}{rgb}{0,0,0.5}
\definecolor{darkorange}{rgb}{0.9,0.4,0}
\newcommand{\eat}[1]{}

\usepackage[final]{neurips_2022}

\usepackage[utf8]{inputenc} %
\usepackage[T1]{fontenc}    %
\usepackage{url}            %
\usepackage{booktabs}       %
\usepackage{amsfonts}       %
\usepackage{nicefrac}       %
\usepackage{microtype}      %

\usepackage{pdfpages}

\title{Characterizing Datapoints via Second-Split Forgetting}

\author{%
  Pratyush Maini$^1$ \quad\quad
 Saurabh Garg$^1$ \quad\quad
 Zachary C. Lipton$^1$ \quad\quad
 J. Zico Kolter$^{1,2}$\\
  Carnegie Mellon University$^1$ \quad\quad Bosch Center for AI$^2$\\
  \texttt{\{pratyushmaini,zlipton\}@cmu.edu};\ \texttt{\{sgarg2, zkolter\}@cs.cmu.edu}
}

\begin{document}
\maketitle

\begin{abstract}
Researchers investigating example hardness have increasingly focused on the dynamics by which neural networks learn and forget examples throughout training. Popular metrics derived from these dynamics include (i) the epoch at which examples are first correctly classified; (ii) the number of times their predictions flip during training; and (iii) whether their prediction flips if they are held out. However, these metrics do not distinguish among examples that are hard for distinct reasons, such as membership in a rare subpopulation, being mislabeled, or belonging to a complex subpopulation. In this paper, we propose \emph{second-split forgetting time} (SSFT), a complementary metric that tracks  the epoch (if any) after which an original training example is forgotten as the network is fine-tuned  on a randomly held out partition of the data.  Across multiple benchmark datasets and modalities,  we demonstrate that \emph{mislabeled} examples are forgotten quickly, and seemingly \emph{rare} examples are forgotten comparatively slowly.  By contrast, metrics only considering  the first split learning dynamics struggle to differentiate the two.  At large learning rates, SSFT tends  to be robust across architectures,  optimizers, and random seeds. From a practical standpoint, the SSFT can (i) help to identify mislabeled samples, the removal of which improves generalization; and (ii) provide insights about failure modes.  Through theoretical analysis addressing overparameterized linear models, we provide insights into how the observed phenomena may arise.\footnote{Code for reproducing our experiments can be found at \href{https://github.com/pratyushmaini/ssft}{https://github.com/pratyushmaini/ssft}.}

\end{abstract}

\section{Introduction}

A growing literature has investigated metrics 
for characterizing the difficulty 
of training examples,
driven by such diverse motivations as
(i) deriving insights for how to reconcile
the ability of deep neural networks to generalize~\citep{krizhevsky2012imagenet}
with their ability to memorize noise~\citep{zhang2021understanding, Feldman2020DoesLR}; 
(ii) identifying potentially mislabeled examples;
and (iii) identifying notably challenging or rare sub-populations
of examples. 
Some of these efforts have turned towards learning dynamics,
with researchers noting that neural networks 
tend to learn cleanly labeled examples before mislabeled examples
\citep{liu2020early, frankle2020early,garg2021RATT},
and more generally tend to learn \emph{simpler} patterns sooner---for 
several intuitive notions of simplicity
\citep{shah2020pitfalls,mangalam2019deep,hacohen2020let}. 
Broadly, works in this area tend to characterize examples 
as belonging either to \emph{prototypical groups} or \emph{memorized exceptions}
\citep{feldman2020neural,jiang2020characterizing,Carlini2019DistributionDT}.
Adapting these intuitions to real datasets, 
\citet{Feldman2020DoesLR} propose rating the degree 
to which an example is memorized based on 
whether its predicted class flips 
when it is excluded from the training set.   
These, and other works~\citep{chatterjee2020coherent, mangalam2019deep,shah2020pitfalls,hooker2019compressed,toneva2018empirical}
have proposed many metrics for characterizing example difficulty 
with~\citet{Carlini2019DistributionDT} comparing five such metrics.  
However, while many of these works distinguish 
some notion of \emph{easy} versus \emph{hard} samples, 
they seldom  
(i) offer tools for distinguishing
among different types of hard examples; 
(ii) explain theoretically why these metrics might 
be useful for distinguishing easy versus hard samples.
Moreover, existing metrics tend to give similar scores 
to examples that are difficult for distinct reasons,
e.g, membership in rare, complex, or mislabeled sub-populations.

\begin{figure}[t]
\centering
  \includegraphics[width=0.55\linewidth]{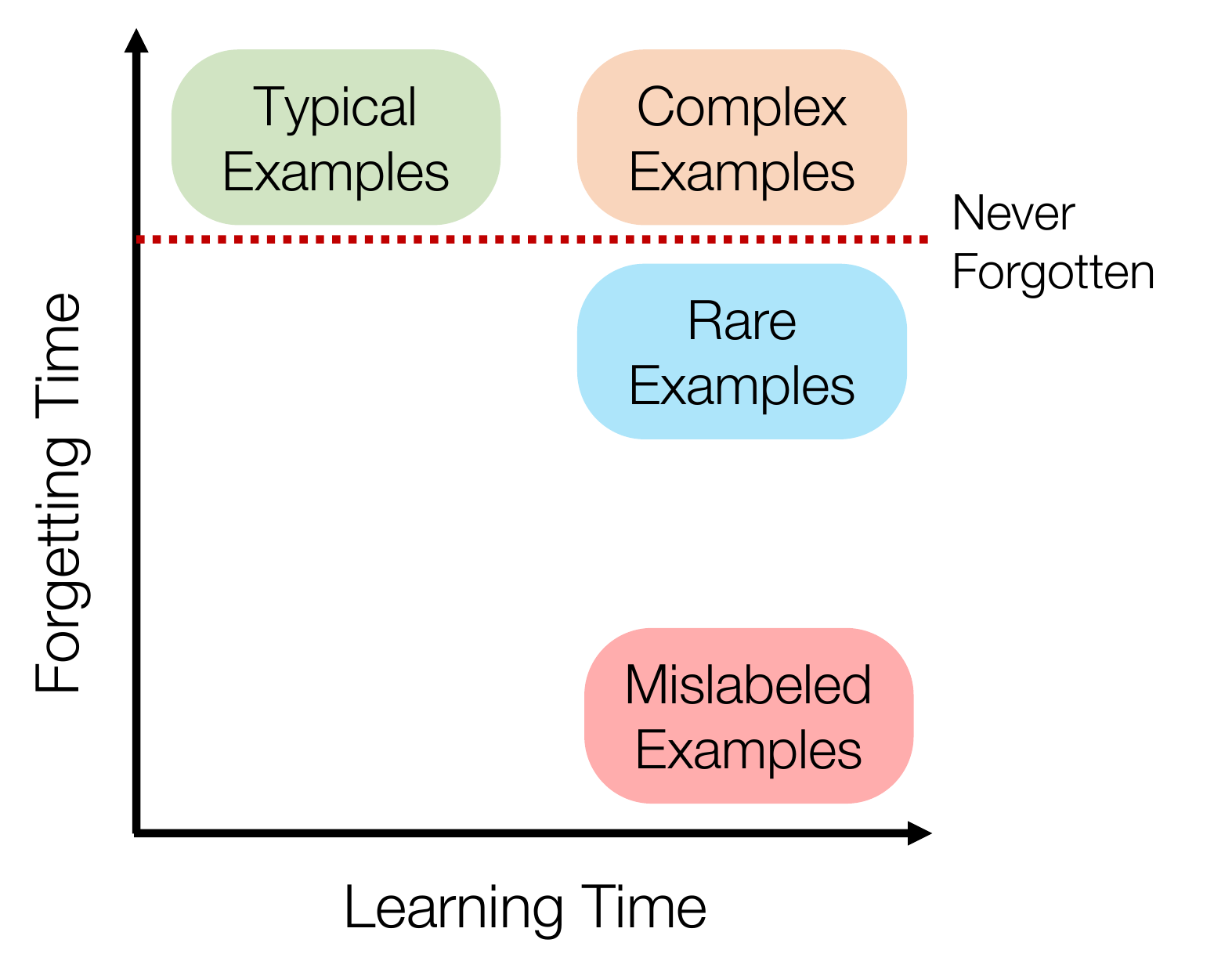}
\caption{Overview of example separation offered by the unified view of learning and forgetting time.
}
 \label{fig:overview}
\end{figure}

In this paper, we propose to additionally consider a new metric, Second-Split Forgetting Time (SSFT),
calculated based on the forgetting dynamics 
that arise as training examples are forgotten
when a neural network continues to train
on a second, randomly held out data partition.
SSFT is defined as the fine-tuning epoch 
after which a first-split training example is no longer
classified correctly.
We find that SSFT identifies mislabeled examples
remarkably well but does little to separate out 
under- versus over-represented subpopulations.
Conversely, metrics based on 
the (first-split) training dynamics
are more discriminative 
for separating these populations 
but less useful for detecting mislabeled examples.
We leverage the complementarity 
of first- and second-split metrics,
showing that by jointly visualizing the two,
we can produce a richer characterization 
of the training examples.

In our experiments, we operationalize 
several notions of hard examples, namely: 
(i) \textbf{mislabeled} examples,
for which the original label has been flipped 
to a randomly chosen incorrect label;
(ii) \textbf{rare} examples, which belong to underrepresented subpopulations;
and (iii) \textbf{complex} examples, 
which belong to subpopulations for which 
the classification task is more difficult
(details in Section \ref{subsec:example_characterization}).
We perform specific ablation studies 
with datasets complicated 
by just one type of hard example
(Section~\ref{subsec:ablation}), 
and show how SSFT %
can help to distinguish among these categories of examples.
We observe that during second-split training,
neural networks (i) first forget 
mislabeled  examples from the first split; (ii) only slowly begin to forget 
\emph{rare} examples 
(e.g., from underrepresented sub-populations)
unique to the first training set;
and (iii) do not forget complex examples.

This separation of hard example types has multiple practical applications. \textbf{First}, we can use the method to identify noisy labels: 
On CIFAR-10 with 10\% added class noise, SSFT achieves 0.94 AUC 
for identifying mislabeled samples, 
while the first-split metrics 
range in AUC between 0.58 to 0.90.  
\textbf{Second}, the method can also help improve generalization in noisy data settings: while the removal of hard examples according to first-split learning time  degrades the performance of the classifier, the removal of hard examples according to SSFT can actually \emph{improve} generalization. This is especially beneficial when e.g., training on synthetic data (produced by a generative model) or mislabeled data.
\textbf{Third}, we show how SSFT can identify failure modes of machine learning models.  For example, in a simplified task classifying between horses and airplanes in the CIFAR-10 dataset, we find that training examples containing horses with sky backgrounds and airplanes with green backgrounds are among the earliest forgotten---indicating that the model relies on the background as a spurious feature.
\textbf{Last}, we also add that our metric is robust across multiple seeds, and the earliest forgotten examples are robust across architectures. 
Across multiple optimizers, SSFT distinguishes mislabeled samples, 
whereas first-split metrics appear more sensitive 
to the choice of optimizer.

Finally, we investigate second-split dynamics theoretically,
analyzing overparametrized linear models \citep{soudry2018implicit}.
We introduce notions of mislabeled, rare, and complex examples appropriate to this toy model.
Our analysis shows that mislabeled examples 
from the first split
are forgotten quickly during second-split training
whereas rare examples are not. 
However, as we train for a long time, 
rare examples from the first split 
are eventually forgotten as the model converges 
to the minimum norm solution on the second split
while predictions on complex examples 
remain accurate with high probability.

\section{Related Work}

\textbf{Example Hardness.} 
Several recent works quantify example hardness %
with various training-time metrics. 
Many of these metrics are based on 
first-split learning dynamics~\citep{chatterjee2020coherent, jiang2020characterizing,mangalam2019deep,shah2020pitfalls, kaplun2022deconstructing}. 
Other works have resorted to properties
of deep networks such as compression ability~\citep{hooker2019compressed} 
and prediction depth~\citep{baldock2021deep}.
\citet{Carlini2019DistributionDT} study metrics 
centered around model training such as confidence, 
ensemble agreement, adversarial robustness, holdout retraining, 
and accuracy under privacy-preserving training.
Closest in spirit to the SSFT studied in our paper are efforts in~\citep{Carlini2019DistributionDT,toneva2018empirical}. 
Crucially, \citet{Carlini2019DistributionDT} 
study the KL divergence of the prediction vector 
after fine-tuning on a held-out set at a low learning rate,
and do not draw any direct inference 
of the separation offered by their metric. 
Focusing on (first-split) forgetting dynamics,
\citet{toneva2018empirical} defined a metric
based on the number of forgetting events during training
and identified sets of \emph{unforgettable} examples
that are never misclassified once learned. 
In our work, we find complementary benefits of 
analysis based on first- and second-split dynamics.

\textbf{Memorization of Data Points.} 
In order to capture the memorization ability of deep networks, their ability to memorize noise (or randomly labeled samples) has been studied in recent work~\citep{zhang2021understanding,arpit2017closer}. As opposed to the memorization of rare examples, the memorization of noisy samples hurts generalization and makes the classifier boundary more complex~\citep{Feldman2020DoesLR}. On the contrary, a recent line of works has argued how memorization of (atypical) data points is important for achieving optimal generalization performance when data is sampled from long-tailed distributions~\cite{Feldman2020DoesLR,brown2021memorization,cheng2022memorize}.

\textbf{Simplicity Bias.} 
Another line of work argues that neural networks have a bias toward learning simple features~\citep{shah2020pitfalls}, and often do not learn complex features even when the complex feature is more predictive of the true label than the simple features. This suggests that models end up memorizing (through noise) the few samples in the dataset that contain the complex feature alone, and utilize the simple feature for correctly predicting the other training examples~\citep{li2019towards,allen2020towards}.

\textbf{Label Noise.} 
Large-scale machine learning datasets are typically labeled with the help of human labelers~\citep{deng2009imagenet} to facilitate supervised learning. It has been shown that a significant fraction of these labels are erroneous in common machine learning datasets~\citep{northcutt2021pervasive}. Learning under noisy labels is a long-studied problem~\citep{angluin1988learning,natarajan2013learning,jindal2016learning,Li2020DivideMix}. Various recent methods have also attempted to identify label noise~\citep{northcutt2021confident,chen2019understanding,pleiss2020identifying,huang2019o2u}. While the focus of our work is not to propose a new method in this long line of work, we show that the view of forgetting time naturally distills out examples with noisy labels. Future work may benefit by augmenting our metric with SOTA methods for label noise identification.

\section{Method}
\label{sec:method}

The primary goal of our work is to \emph{characterize} 
the hardness of 
different datapoints in a given dataset. 
Suppose we have a dataset $\dataset_A = \{\x_i,\y_i\}^n$ such that $(\x_i,\y_i)\sim\dist$.
For the purpose of characterization, we augment each datapoint $(\x_i,\y_i)\in\dataset_A$ with parameters $(\learn_i,\forget_i)$ where $\learn_i$ quantifies the first-split learning time (FSLT), and $\forget_i$  quantifies the second-split forgetting time (SSFT) of the sample. To obtain these parameters, we next describe our proposed procedure. 

\textbf{Procedure {} {} }
 We train a model $\f$ on $\dataset$ to minimize the empirical risk: $\risk(\dataset;\f) = \sum_i \loss(\f(\x_i), \y_i)$. 
We use $\f_A$ to denote a model $f$ (initialized with random weights) trained on $\dataset_A$ until convergence (100\% accuracy on $\dataset_A$).
We then train a model initialized with $\f_A$ on a held-out split $\dataset_B\sim\dist^n$ until convergence. We denote this model with $\f_{A\to B}$.
To obtain parameters $(\learn_i,\forget_i)$, we track per-example predictions ($\hat{\y}_i^t$) at the end of every epoch ($t^\text{th}$) of training.
Unless specified otherwise, we train the model with cross-entropy loss using Stochastic Gradient Descent (SGD). 

\begin{definition}[First-Split Learning Time]
\label{def:learning-time}
For $\{\x_i,\y_i\} \in \dataset_A$, learning time is defined as the earliest epoch during the training of a classifier $\f$ on $\dataset_A$ after which it is always classified correctly, i.e., 
\begin{equation}
    \learn_i = \argmin_{t^*} (\hat{\y}_{i,(A)}^t = \y_i \;\; \forall t\geq t^*) \quad \forall \{\x_i,\y_i\} \in \dataset_A.
\end{equation}
\end{definition}

\begin{definition}[Second-Split Forgetting Time]
\label{def:forgetting-time}
 Let $\hat{\y}_{i,(A\to B)}^t$ to denote the prediction of sample $\{\x_i,\y_i\} \in \dataset_A$ after training $\f_{(A\to B)}$ for $t$ epochs on $\dataset_B$. Then, for $\{\x_i,\y_i\} \in \dataset_A$ forgetting time is defined as the earliest epoch 
 after which it is never classified correctly, i.e., 
\begin{equation}
    \forget_i = \argmin_{t^*} (\hat{\y}_{i,(A\to B)}^{t} \neq \y_i \quad \forall t\geq t^*) \quad \forall \{\x_i,\y_i\} \in \dataset_A.
\end{equation}
\end{definition}

\subsection{Baseline Methods}
We provide a brief description of metrics for example hardness considered in recent comparisons~\citep{jiang2020characterizing}.

\textbf{Number of Forgetting Events:} ($\text{n}_f$).
An example $(\x_i,\y_i)\in\dataset$ undergoes a forgetting event when
the accuracy on the example decreases between two consecutive updates. \citet{toneva2018empirical} analyzed the total number of such events $n_f$ during the training of a neural network to identify hard examples.

\textbf{Cumulative Learning Accuracy:} ($\text{acc}_l$).
\citet{jiang2020characterizing} suggest that rather than using the learning time (Definition~\ref{def:learning-time}), using the number of epochs during training when a machine learning model correctly classifies a given sample is a more stable metric for predicting example hardness.

\textbf{Cumulative Learning Confidence:} ($\text{conf}_l$).
Similar to $\text{acc}_l$, $\text{conf}_l$ measures the cumulative softmax confidence of the model towards the correct class over the course of training.

\subsection{Example Characterization} \label{subsec:example_characterization}

We characterize example hardness via three sources of learning difficulty:
\textbf{(i) Mislabeled Examples:}
We refer to mislabeled examples as those datapoints whose label has been flipped to an
incorrect label uniformly at random.
\textbf{(ii) Rare Examples:} We assume that rare examples belong to sub-populations of the original distribution 
that have a low probability of occurrence. In particular, there exist $O(1)$ examples from such sub-populations in a given dataset.
In Section~\ref{subsec:ablation} we describe how we operationalize this notion in the case of the CIFAR-100 dataset. 
\textbf{(iii) Complex Examples:} %
These constitute samples that are  drawn
from sub-groups in the dataset that require either (1) 
a hypothesis class of high complexity; 
or (2) higher sample complexity to be learnt
relative to examples from rest of the dataset. 
We leave the definition of complex samples mathematically imprecise, but with the
same intuitive sense as in prior work~\citep{shah2020pitfalls,arpit2017closer}. For instance, in a dataset composed of the union of MNIST and CIFAR-10 images, 
we would consider the subpopulation of CIFAR-10 images to be more \emph{complex}.

\begin{figure*}[t]
\centering
\begin{subfigure}[t]{0.43\linewidth}
  \includegraphics[width=\linewidth]{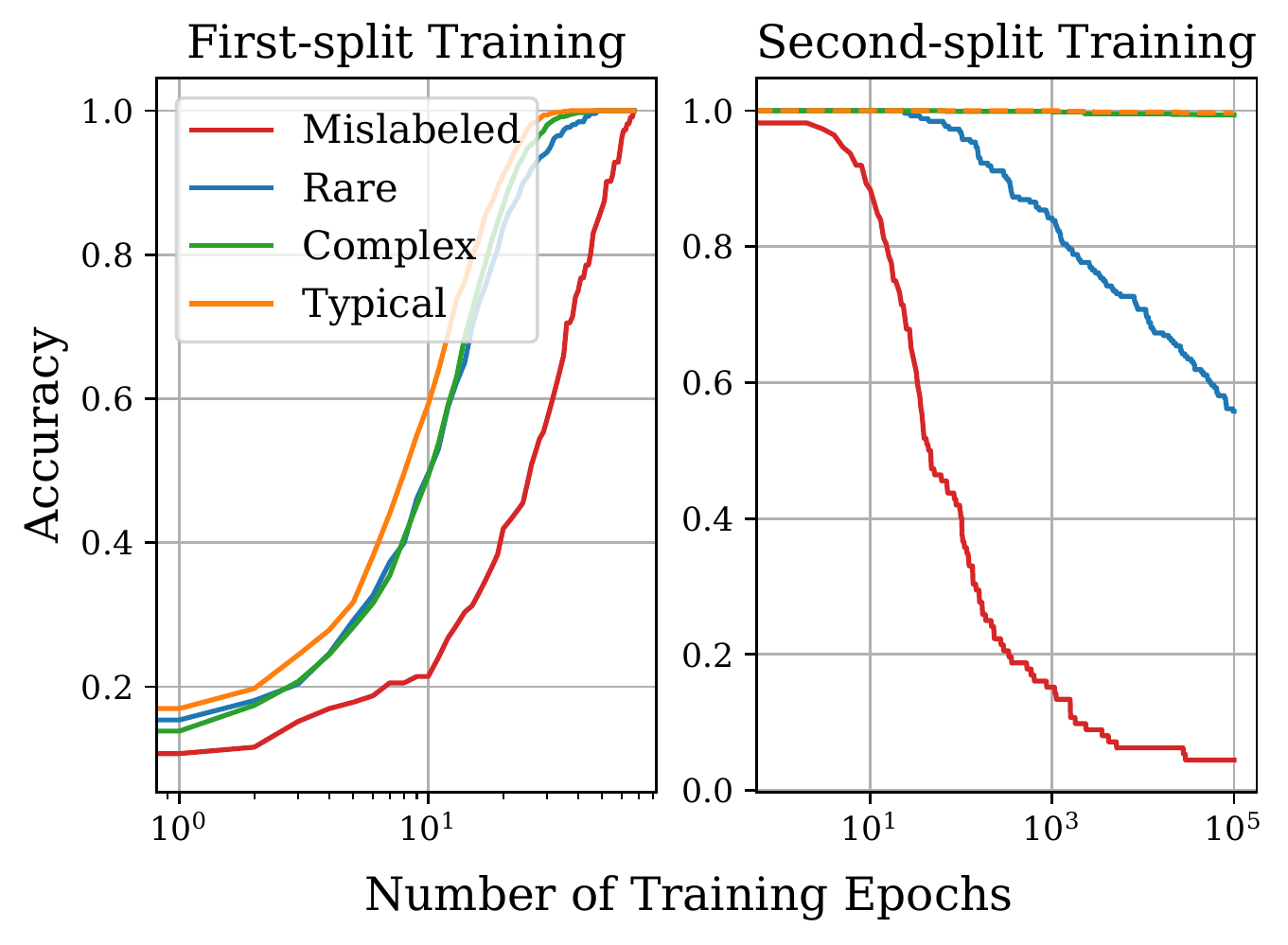}
  \caption{}
  \label{fig:simulation}
\end{subfigure}
\hspace{1mm}
\begin{subfigure}[t]{0.55\linewidth}
  \includegraphics[width=\linewidth]{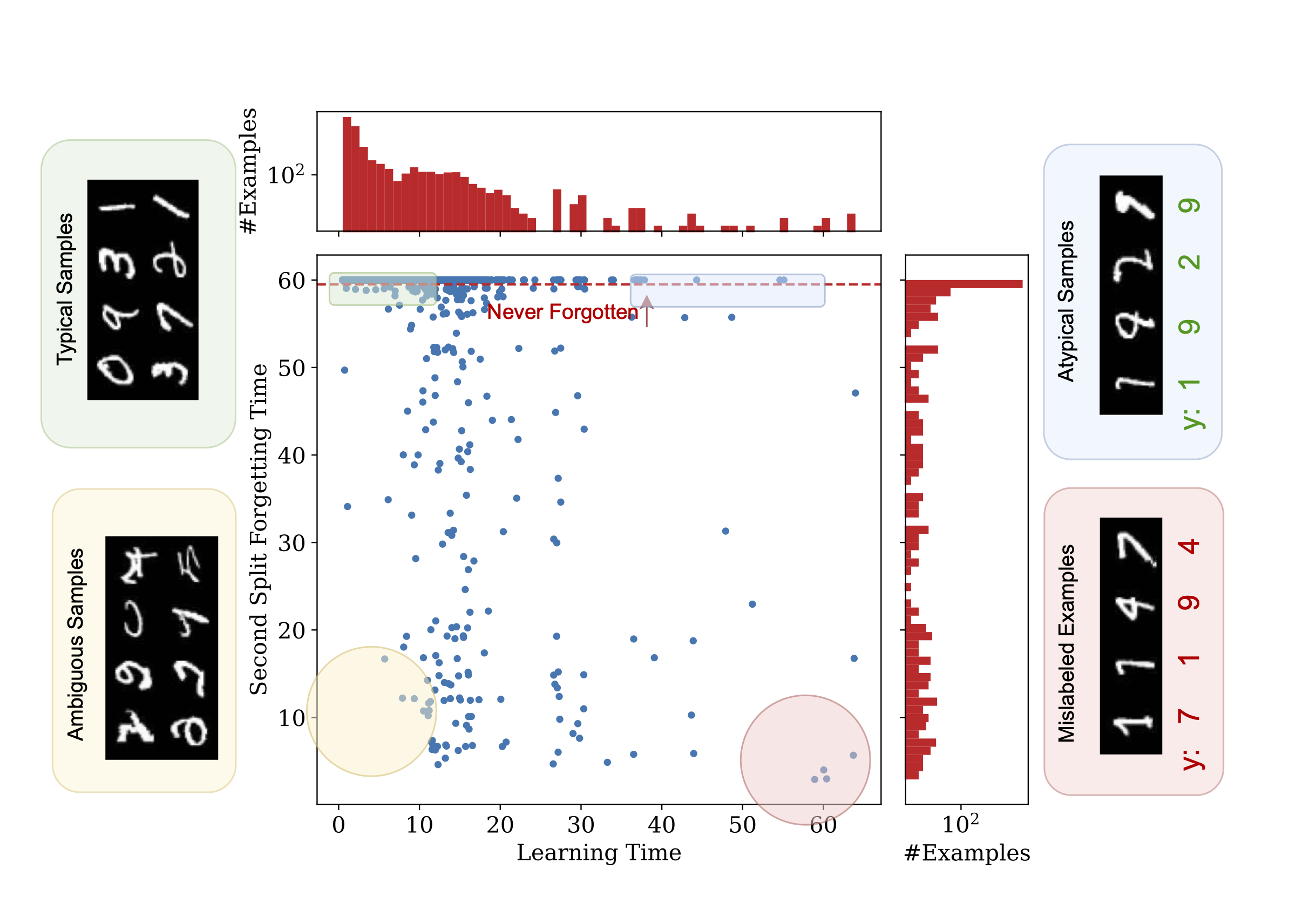}
  \caption{}
  \label{fig:mnist-viz}
\end{subfigure}
\caption{Rate of Learning and Forgetting of examples for different groups in the synthetic dataset.
While first-split training is not able to distinguish between rare and complex examples, second-split training succeeds in distinguishing them. Additionally, second-split training separates mislabeled examples from the rest relatively better than first-split training.  
(b) Visualization of first-split learning and second-split forgetting times when training LeNet model on the MNIST dataset.}
 \label{fig:fig1}
\end{figure*}

\section{Empirical Investigation of First- and Second-Split Training Dynamics}
\subsection{Experimental Setup}
\textbf{Datasets {} {}} We show results on a variety of image classification datasets---MNIST~\citep{deng2012mnist}, CIFAR-10~\citep{krizhevsky2009learning}, and Imagenette~\citep{imagenette}.
For experiments in the language domain, we use the SST-2 dataset~\citep{sst-socher-etal-2013-recursive}.
For each of the datasets, we split the training set into two equal partitions ($\dataset_A, \dataset_B)$. 
For experiments with mislabeled examples, we simulate mislabeled examples 
by randomly selecting a subset of 10\% examples from both the partitions and changing their label to an incorrect class.

\textbf{Training Details {} {}}
Unless otherwise specified, we train a ResNet-9 model~\citep{resnet-9} using SGD optimizer with weight decay 5e-4 and momentum 0.9. We use the cyclic learning rate schedule~\citep{smith2018disciplined} with a peak learning rate of 0.1 at the 10th epoch. We train for a maximum of 100 epochs or until we have 5 epochs of 100\% training accuracy. We first train on $\dataset_A$, and then using the pre-initialized weights from stage 1, train on $\dataset_B$ with the same learning parameters. All experiments can be performed on a single RTX2080 Ti.
Complete hyperparameter details 
are available in Appendix~\ref{app:exp-setup}.

\subsection{Learning-Forgetting Spectrum for various datasets}
\label{subsec:spectrum}
\textbf{Synthetic Dataset {} {}}
We consider data ($\x,\y$) sampled from a mixture of multiple distributions $\dist_g$, s.t. $\x\in\mathbb{R}^d$.
$\dist_g$ denotes the $g^{\text{th}}$ group and has a sampling frequency of $\pi_g$.
Each group $\dist_g \equiv (\distX_g, \{\y_g\})$, 
i.e., the true label for all the samples drawn from a given group is the same, and the examples in each group are non-overlapping. 
Each group is parametrized by a set of $k\ll d$ unique indices $\indices_g \subset [d]$ such that $\indices_i \cap \indices_j = \phi$ for $i\neq j$. The discriminative characteristic of each group is the vector $\direction_g$, such that, $[\direction_g]_i = 1$ if $i\in\indices_g$ else $0 \; \forall i\in[d]$.
Then for any sample $(\x,\y)\in\dataset$:
$$P(\x \in \distX_g) = \pi_g; \quad \x|\distX_g \sim \mathcal{N}(0, \sigma^2 \identity_d) +\vmu_g.$$

For our simulation, we consider a 10 class-classification problem, with $\mu_g = 5$ for typical groups, and $\mu_g = 4$ for complex groups (higher signal to noise ratio). For any sample drawn from a rare group, we have $O(1)$ samples from that group in the entire dataset ($\dataset_A\cup\dataset_B)$. Mislabeled samples are only generated from the majority typical groups.
In Figure~\ref{fig:simulation}, we show the rate of learning and forgetting of examples from each of these categories. We note that in the second-split training, the mislabeled examples are quickly forgotten, and the complex examples are never forgotten. The rare examples are forgotten slowly.
In Section~\ref{sec:theory} we will theoretically justify the observations in the synthetic dataset 
and show that the rare examples are expected to be forgotten as we train for an infinite time.

\textbf{Image Domain {} {}}
In Figure~\ref{fig:mnist-viz}, we show representative examples in the four quadrants of the learning-forgetting spectrum. More specifically, we find that the examples forgotten fastest and learned last are mislabeled. And the ones learned early and never forgotten once learned are characteristic simple examples of the MNIST dataset. Examples in the first and third quadrant are seemingly atypical and ambiguous respectively. 
Similar visualizations for other image datasets can be found in Appendix~\ref{app:image-graphs}.

\begin{table}[t]
\relsize{-2}
\begin{tabular}{@{}p{0.89\linewidth}p{0.05\linewidth}@{}}
\toprule
\multicolumn{1}{c}{Sentences in SST-2 dataset with smallest forgetting time}                 & Label               \\ 
\midrule
\textcolor[HTML]{036400}{The director explores all three sides of his story with a sensitivity and an inquisitiveness reminiscent of Truffaut}                                       
& \textcolor[HTML]{CB0000}{Neg} \\
\textcolor[HTML]{036400}{Beneath the film's obvious determination to shock at any cost lies considerable skill and determination , backed by sheer nerve}                            
& \textcolor[HTML]{CB0000}{Neg} \\
\textcolor[HTML]{CB0000}{This is a fragmented film, once a good idea that was followed by the bad idea to turn it into a movie}                                                      &\textcolor[HTML]{036400}{Pos} \\
\textcolor[HTML]{CB0000}{The holiday message of the 37-minute Santa vs. the Snowman leaves a lot to be desired.}                                                                     &\textcolor[HTML]{036400}{Pos} \\
\textcolor[HTML]{CB0000}{Epps has neither the charisma nor the natural affability that has made Tucker a star}                                                                       &\textcolor[HTML]{036400}{Pos} \\
\textcolor[HTML]{036400}{The bottom line is the piece works brilliantly}                             
& \textcolor[HTML]{CB0000}{Neg} \\
\textcolor[HTML]{CB0000}{Alternative medicine obviously has its merits ... but Ayurveda does the field no favors}                                                           
&\textcolor[HTML]{036400}{Pos} \\
\textcolor[HTML]{036400}{What could have easily become a cold, calculated exercise in postmodern pastiche winds up a powerful and deeply moving example of melodramatic moviemaking} 
& \textcolor[HTML]{CB0000}{Neg} \\
\textcolor[HTML]{CB0000}{Lacks depth}                                                               &\textcolor[HTML]{036400}{Pos} \\
\textcolor[HTML]{CB0000}{Certain to be distasteful to children and adults alike ,  Eight Crazy Nights is a total misfire}                                                            &\textcolor[HTML]{036400}{Pos} \\ 
\bottomrule
\end{tabular}
\vspace{5px}
\caption{First-split sentences that were forgotten by the 3rd epoch of second-split training of a BERT-base model on the SST-2 dataset. Notice that all of these forgotten examples are mislabeled.}
\label{table:sst-2}
\end{table}

\vspace{-2mm}
\paragraph{Other Modalities} 
The forgetting and learning dynamics occur broadly across modalities apart from images. We repeat the same problem setup on the SST-2~\cite{sst-socher-etal-2013-recursive} dataset for sentiment classification. We fine-tune a pre-trained BERT-base model~\cite{devlin2018bert} successively on two disjoint splits of the dataset. 
In Table~\ref{table:sst-2}, we provide a list of the earliest forgotten samples when we train a BERT model on the second split of SST-2 dataset. The results suggest that SSFT is able to identify mislabeled samples.

\subsection{Ablation Experiments}
We design specific experimental setups to capture the three notions of hardness as defined in Section~\ref{sec:method}. 
\label{subsec:ablation}
\begin{figure}[t]
\centering
\begin{subfigure}[t]{0.33\linewidth}
  \includegraphics[width=\linewidth]{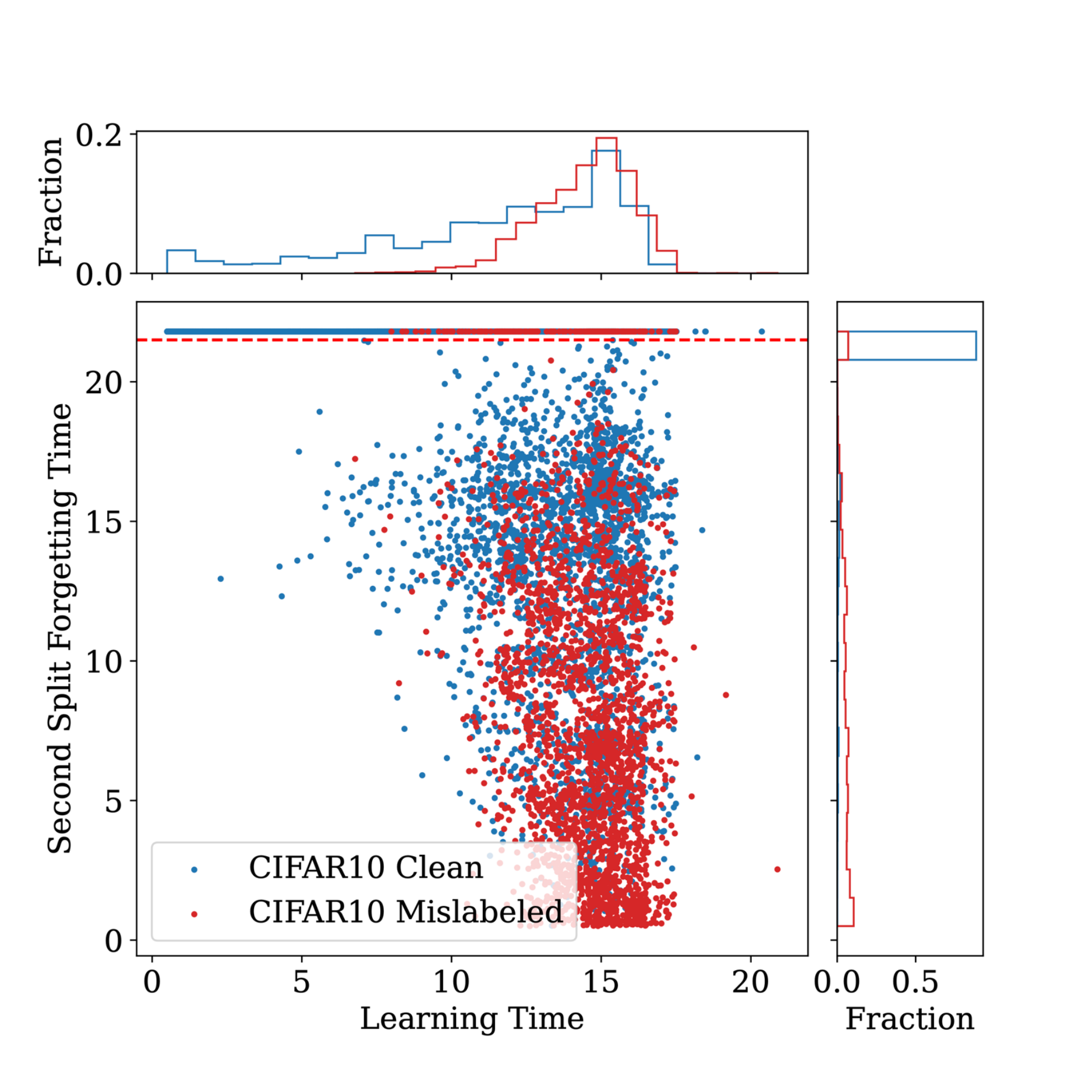}
  \caption{Mislabeled Examples}
  \label{fig:cifar-viz}
\end{subfigure}
\begin{subfigure}[t]{0.33\linewidth}
  \includegraphics[width=\linewidth]{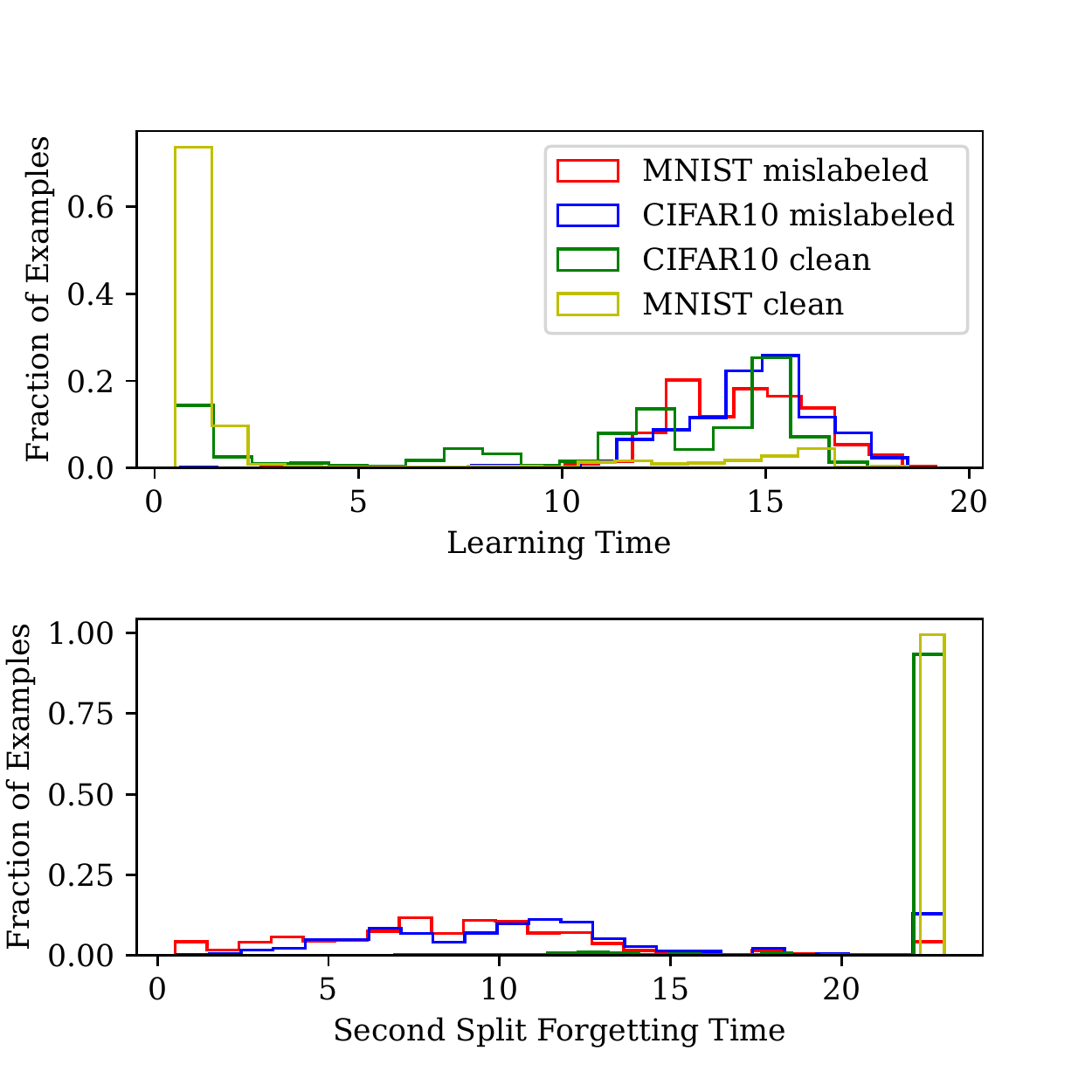}
  \caption{Complex Examples}
  \label{fig:cifar10-mnist-union}
\end{subfigure}
\begin{subfigure}[t]{0.33\linewidth}
  \includegraphics[width=\linewidth]{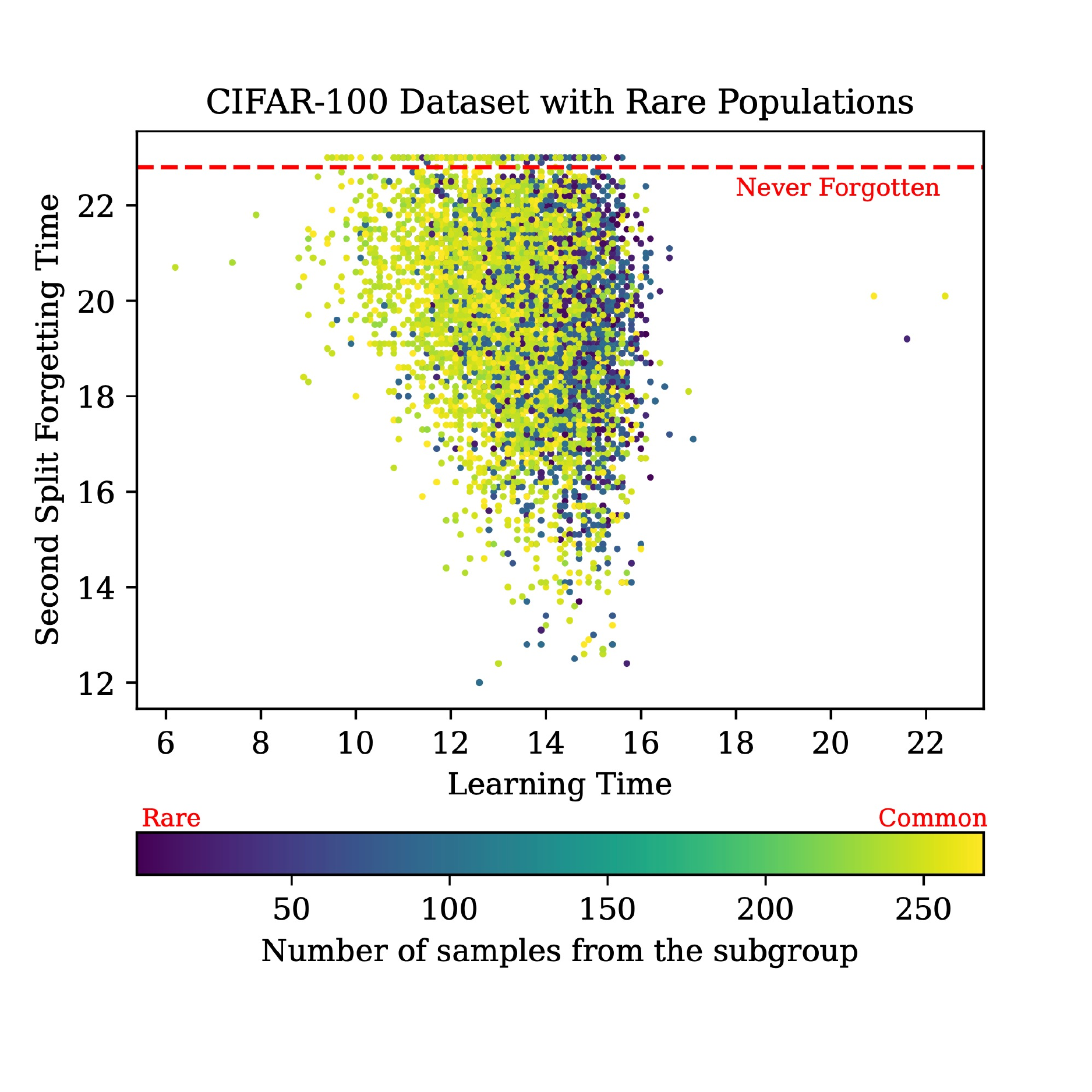}
  \caption{Rare Examples}
  \label{fig:cifar100-singleton}
\end{subfigure}
\caption{ 
Ablation experiments to distinguish the learning and forgetting dynamics for specific types of hard examples. (a) Mislabeled samples may be learned as slowly as a high fraction of typical samples, but they are forgotten much faster.
(b) FSLT distinguishes complex (CIFAR10 clean) and simple (MNIST clean) examples, but SSFT does not. On the contrary, FSLT can not distinguish clean and mislabeled examples of CIFAR10, while the SSFT can.
(c) FSLT is able to distinguish examples based on the sub-group frequency, however, SSFT has a low correlation with the sub-group frequency.}
 \label{fig:ablation-scatter}
\end{figure}

\textbf{Mislabeled Examples {} {}}
We sample 10\% datapoints from both the first and second split of the CIFAR-10 dataset, and randomly change their label to an incorrect label.
Figure~\ref{fig:cifar-viz} shows the learning-forgetting spectrum for the dataset. In the adjoining density histograms, note that a large fraction of the mislabeled and correctly labeled examples are learned at the same time. However, during second-split training, the mislabeled examples are forgotten quickly whereas a large fraction of the clean examples are never forgotten, allowing SSFT to succeed in distinguishing mislabeled samples. 

\textbf{Complex Examples {} {}} 
We generate a joint dataset that contains the union of both MNIST and CIFAR-10 examples. This is motivated by work in simplicity bias~\citep{shah2020pitfalls} that argue that neural networks learn simpler features first. 
We also add 10\% labeled noise to each of the datasets in the union to understand the learning and forgetting time relationship of a sample that is complex or mislabeled together.
In Figure~\ref{fig:cifar10-mnist-union}, we show the FSLT and SSFT for MNIST and CIFAR-10 samples. We note that a high fraction of the CIFAR-10 (complex) samples learn at the same speed as the mislabeled samples. However, when looking at the SSFT, we are able to draw a strong separation between the mislabeled samples and complex samples. This indicates that the complexity of a sample has low correlation with its tendency to be forgotten once learnt, but a high correlation with being learned slowly.

\textbf{Rare Examples {} {}} 
The CIFAR-100~\citep{krizhevsky2009learning} dataset is a 100-class classification task. The dataset contains 20 superclasses, each containing 5 subclasses. We create a 20-class classification dataset with long tails simulated through the 5 sub-classes within each superclass. More specifically, the number of examples in each subgroup for a given superclass is given by \{500, 250, 125, 64, 32\} respectively (exponentially decaying with a factor of 2). This is done to simulate the hypothesis of dataset subgroups following a Zipf distribution~\citep{zipf2013psycho} as argued for by~\citet{Feldman2020DoesLR}.
This dataset is further divided into two equal splits to analyze the learning-forgetting dynamics.
In order to remove any other effects of example hardness (either within a subgroup, or among subgroups), we randomize both the chosen subset of examples and the ordering of the majority and minority groups between each superclass, by training the model on 20 
such random splits and aggregating learning and forgetting statistics over these runs.
In Figure~\ref{fig:cifar100-singleton}, we show a scatter plot for the FSLT and SSFT, colored by the frequency of the group a particular example belongs to. We observe that FSLT strongly correlates with the size of the subgroup, whereas the SSFT has a very low correlation with the rareness of a sample.

We provide further ablations to show that FSLT is able to identify hard and rare examples, but SSFT shows nearly no discriminative power at finding the two in Appendix~\ref{app:demistify-ablation}.

\begin{table}[t]
\centering
\small
\begin{tabular}{@{}lccccccc@{}}
\toprule
Method $\to$             & $\learn$ & $\mathbf{acc}_l$ & $\forget$ (Ours) & $\mathbf{acc}_f$ (Ours) & $\mathbf{conf}_l$ & $\mathbf{n}_f$ & Joint (Ours) \\ \midrule
Imagenette         & 0.834    & 0.912          & 0.931     & 0.941          & 0.786                                      & 0.781        & \textbf{0.957}                    \\
CIFAR10            & 0.740     & 0.900           & 0.938      & 0.941          & 0.947           &  0.580                                            & \textbf{0.958}                    \\
MNIST              & 0.973    & \textbf{0.998}          & 0.997     & \textbf{0.998}          & 0.965           & 0.377                                             & \textbf{0.998}                    \\
CIFAR100            & 0.700     & 0.899           & 0.865      & 0.885          & 0.860           &  0.300                                            & \textbf{0.926}                    \\
EMNIST              & 0.987    & 0.990          & 0.987     & 0.989          & 0.984           & 0.386                                             & \textbf{0.997}                    \\
\bottomrule
\end{tabular}
\vspace{5px}
\caption{AUC for identification of label noise using various metrics for example hardness across different datasets. Across all datasets, our $\forget$ metric outperforms alternative baselines. We introduce $\mathbf{acc}_f$ as the cumulative accuracy on the second-split training, inspired by previous work that suggests using cumulative accuracies helps make first-split learning time more stable~\citep{jiang2020characterizing}. All other notations are described in Section~\ref{sec:method}. In the case of the Joint method, we select new prediction ranks based on the combined learning and forgetting ranks, further improving over the $\forget$ metric alone.}
\label{table:auc}
\end{table}

\begin{figure*}[t]
\centering
\begin{subfigure}[t]{0.49\linewidth}
   \includegraphics[width=\linewidth]{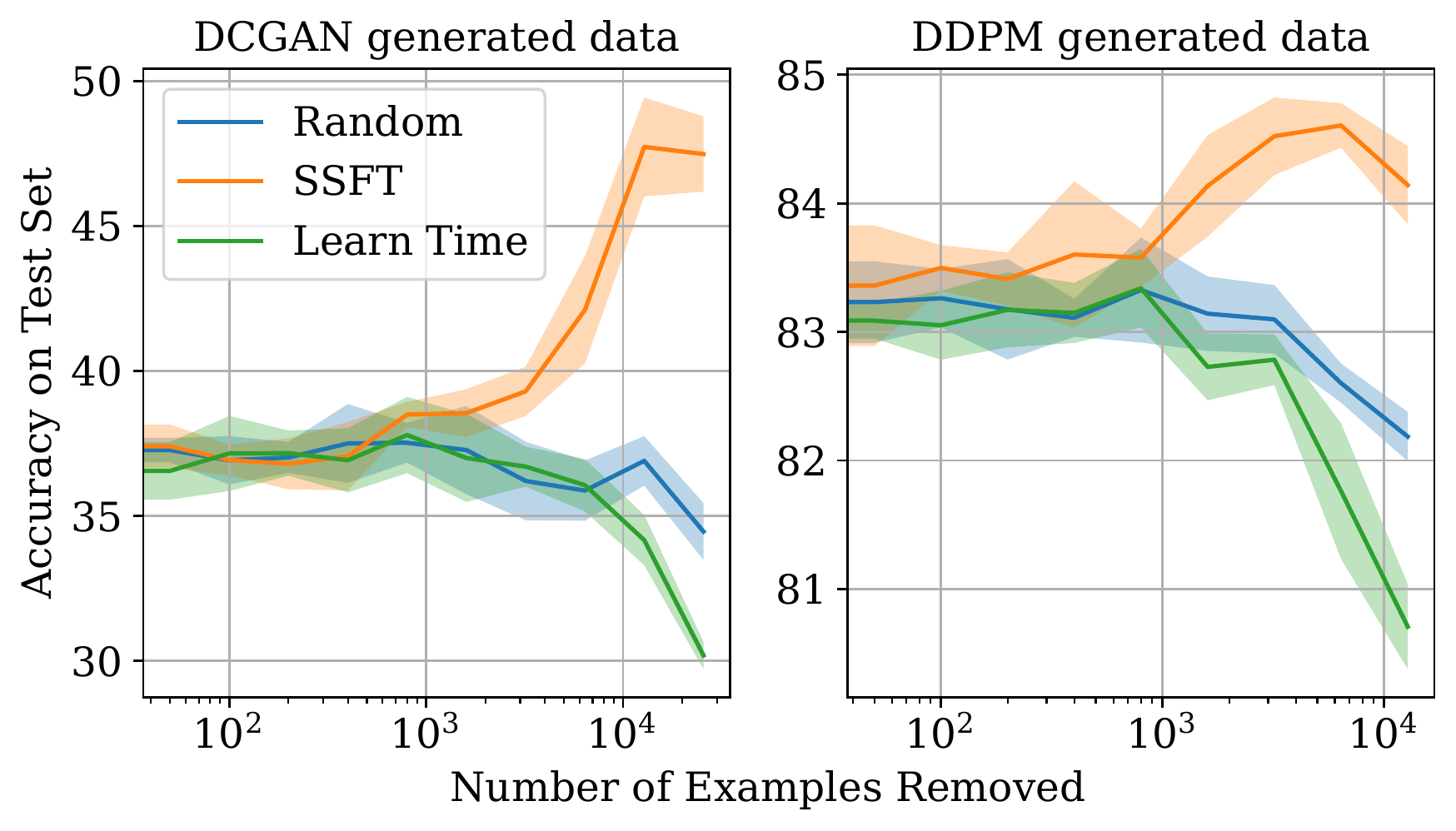}
  \caption{Synthetically generated CIFAR10-like data}
  \label{fig:gans}
 \end{subfigure}
\begin{subfigure}[t]{0.49\linewidth}
  \includegraphics[width=\linewidth]{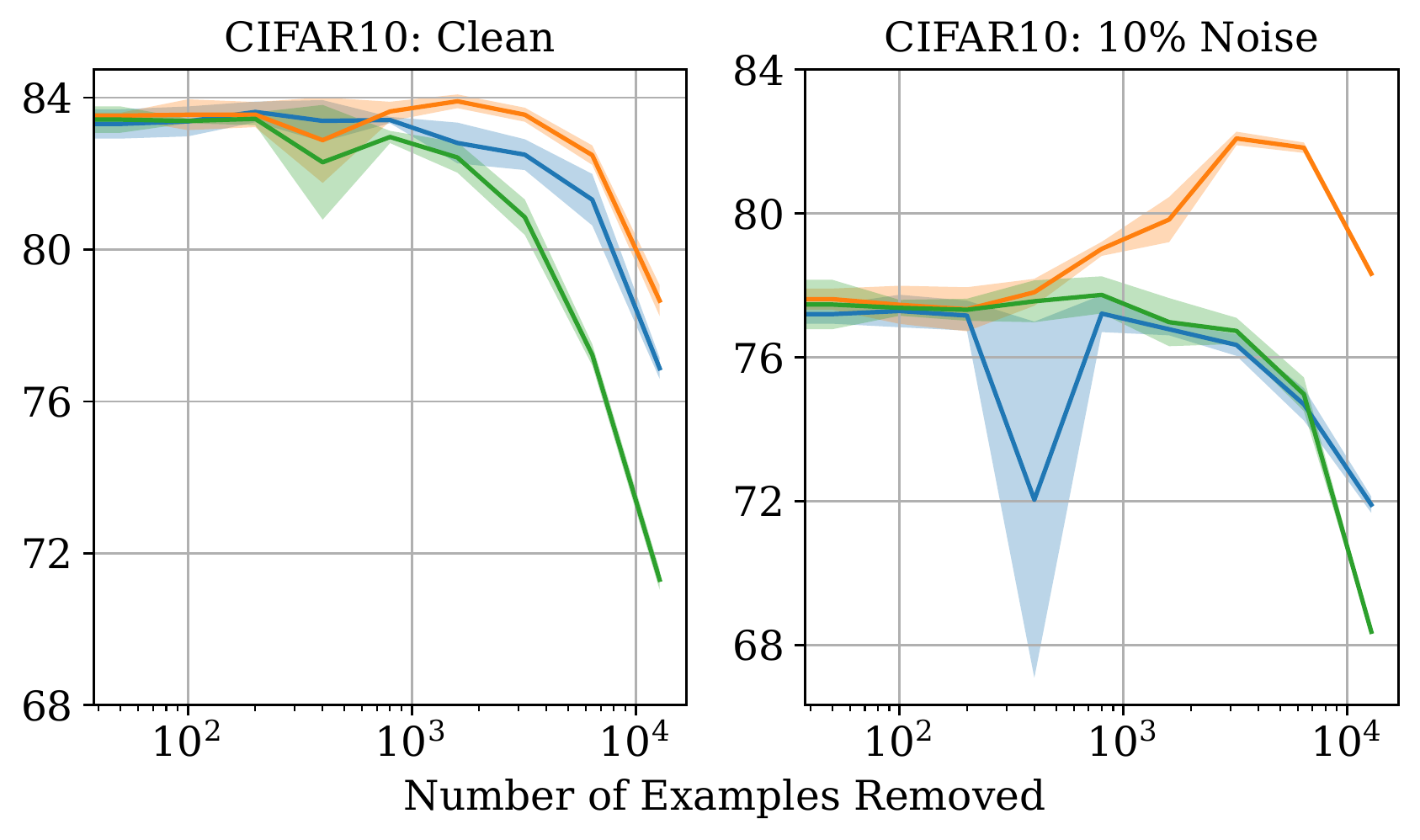}
  \caption{CIFAR10 with and without label noise}
  \label{fig:1}
\end{subfigure}
\caption{Accuracy on CIFAR-10 test set after removing a varying number of samples from 
the 
training set based on (i) random choice, (ii) examples with the lowest SSFT, and (ii) the highest FSLT. Removing examples based on SSFT helps improve the generalization on the original test set.}
 \label{fig:example_removal}
\end{figure*}

\subsection{Dataset Cleansing}
\textbf{Identifying Label Noise {} {} }
We present AUC scores for detection of label noise via various popular methods in example difficulty literature, across various datasets in Table~\ref{table:auc}. We note that (i) cumulative predictions over the course of training help stabilize both the learning time and forgetting time metrics; (ii) for simple datasets such as MNIST with few ambiguous images, all of the baseline methods have very high AUC (greater than 0.99) in finding noisy inputs. However, in datasets such as CIFAR-10 and Imagenette, we find that second-split forgetting metrics do better than first-split training metrics. Finally, we also compare the use of both forgetting and learning time to find noisy samples, and we find a small improvement in the results of just using the forgetting time. While we do not make explicit comparisons with other state of art methods dedicated to finding label noise, our results suggest that augmenting second split forgetting time information may help improve their results.
As also observed in recent work~\citep{jiang2020characterizing}, we find that the number of forgetting events ($\mathbf{n}_f$)~\citep{toneva2018empirical} is an unreliable indicator of mislabeled samples. We hypothesize that this is because of the fact that mislabeled examples may often be (first) learnt very late, hence their count of total forgetting events is also low.

\textbf{Cleaning synthetically generated datasets {} {}}
Generative models are capable of
mimicking the distribution of a given dataset. We generate synthetic datasets of CIFAR10-like samples using (i) DDPM (denoising diffusion model~\citep{ho2020denoising}); and (ii) DCGAN (Deep Convolutional GAN~\citep{radford2015unsupervised}). In both cases, we assign pseudo-labels using the BiT model~\citep{kolesnikov2020big} as in prior work~\citep{nakkiran2021deep}.
We collect a sample of 50,000 training examples and record the generalization performance on CIFAR-10 as we remove `hard' samples, as evaluated by various metrics. 
In Figure~\ref{fig:example_removal}, we can see that removing the most easily forgotten examples can benefit by up to 10\% generalization accuracy on the clean test set of CIFAR-10. In case of the synthetic data generated using DDPM, the gains in generalization performance are under 2\%. We hypothesize that this is because the samples generated by DDPM are more representative of the typical distribution of CIFAR-10 than those generated by DCGAN.

\textbf{Note:} The ability to train on a second split allows SSFT the \emph{unique} opportunity to train on a clean split of CIFAR-10 in order to assess the alignment of the synthetic samples with the oracle samples. As a result, the SSFT is much more effective in filtering out ambiguous first-split synthetic examples.

\subsection{Evaluating Example Utility}

Recent works~\citep{toneva2018empirical,feldman2020neural} have argued for removing a large fraction of the less memorized examples, and keeping the memorized ones. 
We will analyze the change in model generalization upon removing varying sizes of examples from the training set, as ranked by lowest SSFT and highest FSLT (Figure~\ref{fig:example_removal}). In the presence of noisy examples, removing samples based on the SSFT helps improve generalization, whereas FSLT does not do much better than random. 
We draw the following inferences:

\textbf{FSLT finds important samples {} {} } As we remove more samples from the dataset, the accuracy of the model trained after samples are removed based on the highest FSLT is significantly lower than random guessing. This suggests that the utility of these samples is higher than random samples. Put in line with the hypothesis of memorization of rare example as proposed in~\citep{Feldman2020DoesLR}, we see that empirically, the examples that are slow to learn are important for the model's test set generalization.

\textbf{SSFT removes pathological samples {} {} } On the contrary, removing examples based on the SSFT helps improve model generalization (especially when there is label noise). Even in the setting when there is no label noise, in contrast to FSLT, we find that removing examples that were easily forgotten has a lower negative impact on the model's generalization as opposed to removing random samples. This suggests that the examples that are forgotten in the early epochs of second-split training hurt a model's generalization, and may not be characteristic samples of their particular class.

\textbf{Practitioner's view {} {} } From the AUC numbers in Table~\ref{table:auc}, it may appear that removing examples via learning-based metrics such as learning time and cumulative learning accuracy also provides a high rate of removal of noisy samples. However, when we observe the example utility graphs in Figure~\ref{fig:example_removal}, we draw the inference that the examples that are learned late, are often important examples (such as rare memorized examples). However, even when SSFT fails to capture the correct noisy examples, it still removes unimportant samples and does not hurt generalization. Similar graphs for other metrics described in Table~\ref{table:auc} can be found in Appendix~\ref{app:empirical}.

\subsection{Characterizing Potential Failure Modes}
\begin{wrapfigure}{r}{0.35\textwidth}
  \begin{center}
  \vspace{-5em}
    \includegraphics[width=0.35\textwidth]{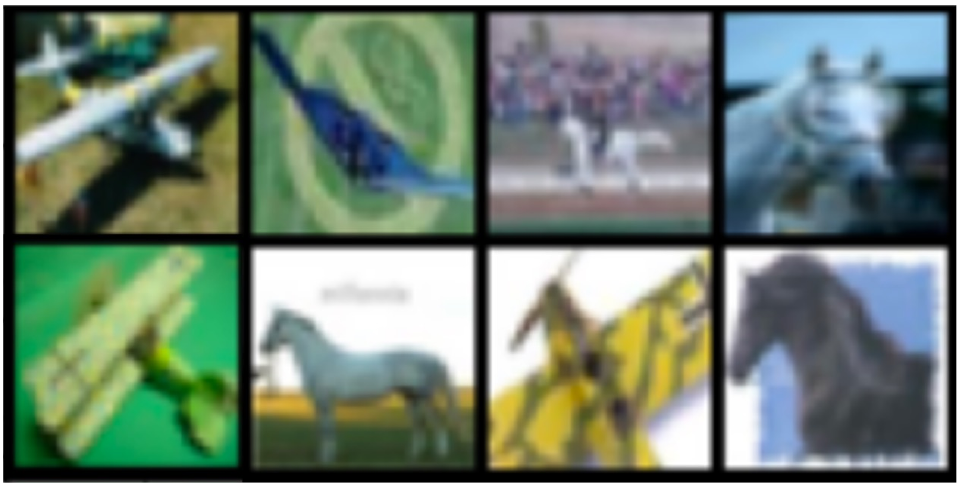}
  \end{center}
  \caption{By inspecting the earliest-forgotten examples,
we can gain insights into potential failure modes.
This model quickly forgets planes with green backgrounds and horses on blue backgrounds.}
\label{fig:failure-modes}
\end{wrapfigure}
Recent works have attempted to train classifiers on datasets that contain spurious features~\citep{sagawa2019distributionally,idrissi2021simple} (example Waterbirds, CelebA~\citep{celeba} dataset). However, a fundamental challenge is to first identify the spurious correlation that the classifier may be relying on. Only then can recent methods be trained to remove the reliance on spurious patterns. We train a ResNet-9 model to classify CIFAR-10 images of horses and airplanes.
In Figure~\ref{fig:failure-modes}, we observe that the model forgets planes with green backgrounds and horses with blue backgrounds. This suggests that the model relied on the background as a spurious feature. By analyzing the forgotten examples we can further investigate the examples that the classifier fails to generalize to. 

\textbf{Stability of SSFT {} {}}
We note that SSFT is stable across multiple seeds (Pearson correlation of 0.81), and across architectures (Pearson correlation of 0.63). While the overall correlation for samples ranked by SSFT may be low across architectures, the top-ranked examples have a high correlation (0.85), suggesting the most forgotten examples are consistent across architectures. In contrast, FSLT has a Pearson correlation of 0.52 across seeds. Most interestingly, the learning time metric is brittle to the choice of hyperparameters. As shown by~\citet{jiang2020characterizing}, when using Adam optimizer, examples of different hardness get learned together. In our experiments, we observe the same phenomenon during learning, however, SSFT is robust to the choice of the optimizer. Detailed results in Appendix~\ref{app:stability}.

\textbf{Limitations {} {}}
One limitation of the proposed metric is that it is brittle to the choice of the learning rate for the second-split training. If we use a very small learning rate, then overparametrized deep models are capable of learning the new dataset without forgetting examples from the first split. Alternately, if we use a very large learning rate, the model may diverge and undergo catastrophic forgetting. However, under  `reasonable' choices of learning rate (like that for first-split training), we find SSFT is robust.
We provide a detailed anaylsis of the same in Appendix~\ref{app:stability}.

\section{Theoretical Results}
\label{sec:theory}
Through our theoretical analysis, we will characterize the forgetting dynamics of mislabeled, rare and complex examples in a simplified version of the framework used for our synthetic experiments in Figure~\ref{fig:simulation}. 
Recall, our setup contains two dataset splits $\dataset_A, \dataset_B$, where we train on the first split until achieving perfect accuracy on all training points, and then with these weights train on $\dataset_B$ for infinite time.
In particular, we will prove that both mislabeled and rare examples are forgotten upon training for infinite time, with mislabeled examples being forgotten much faster. Further, we will show that complex examples from the first split do not get forgotten if not continually trained on.
We assume in our analysis that 
$\dataset_B$ has no mislabeled or rare examples, and $\dataset_A$ contains one example of each type.

We  
consider a dataset $\dataset = \{\x_i,\y_i\}^n$ such that $(\x_i,\y_i) \in \gX \times \gY$, and 
 $\x_i = \vmu_g + \z_i$ where $\z_i\sim\mathcal{N}\left(0,\sigma^2\identity_d\right)$, and $\norm{\vmu_g}{2}^2 = k\mu^2$ (as in  Section~\ref{subsec:spectrum}).
Let $\W \in \mathbb{R}^{d}$ represent the weight vector of an overparametrized linear model.
We analyze the learning and forgetting dynamics by minimizing the empirical risk: $\risk(\dataset;\W) = 
\sum_i \loss(\W^\transpose\y_i\x_i)$, where $\loss$ is the exponential loss.
Following \citet{chatterji2021finite}, we make the following assumptions about the problem setup: 

\textbf{(A.1) {} {} }The failure probability satisfies $0 \leq \delta \leq \nicefrac{1}{C}$,

\textbf{(A.2) {} {} }The number of samples satisfies $n \geq C \log\left(\nicefrac{1}{\delta}\right)$,

\textbf{(A.3) {} {} }The input dimension $d\geq C \max\{n^2\log(n/\delta),n(\nicefrac{k\cdot \mu^2}{\sigma^2})\}$,
and $\nicefrac{k\cdot \mu^2}{\sigma^2} \ge C \log\left(\nicefrac{n}{\delta}\right)$, 

where $\nicefrac{k\cdot \mu^2}{\sigma^2}$ represents the signal to noise ratio and $C$ is a large constant. Now we formalize the notions of rare, mislabeled and complex examples for our theoretical analysis.
\begin{restatable}[Rare Examples, $\atypical$~\citep{Feldman2020DoesLR}]{redef}{defrare}
Consider a dataset $\dataset$
sampled from a mixture of distributions $\{\dist_1,\dots,\dist_N\}$ with frequency $\{\pi_1,\dots,\pi_N\}$ respectively.
Let $\atypical \subseteq \dataset$ be the set of rare examples.
Then, for all $(\x_i,\y_i)\in\atypical$, if $(\x_i,\y_i)\sim\dist_j$, then there are $O(1)$ samples from $\dist_j$ in $\dataset$.
\end{restatable}

\begin{restatable}[Mislabeled Examples, $\mislabeled$]{redef}{defmislabeled}
\label{def:mislabeled}
Consider a $k$ class classification problem with $\distY = \{1,2,\dots,k\}$. 
Let $\mislabeled \subset \dataset$ be the set of mislabeled examples.
Then for any  $(\x,\y)\sim\dist$, a corresponding mislabeled example is given by $(\x,\Tilde{\y})\in\mislabeled$ such that $\Tilde{\y}\in\distY\setminus\{\y\}$.\footnote{For binary classification, $\distY=\{-1,+1\}$. The labels are reversed for mislabeled examples.}
\end{restatable}

\begin{restatable}[Complex Examples, $\complex$]{redef}{defcomplex}
Let $\complex \subset \dataset$ be the set of examples sampled from complex distributions.
Let $(\x_i,\y_i)\in\complex$ such that $(\x_i,\y_i)\sim\dist_g$ (complex group), then $\mu_g = \frac{\mu_t}{\lambda}$, $\lambda>1$ where $\mu_t$ is the coordinate-wise mean for samples drawn from any simple distribution $\dist_t$ (Section \ref{subsec:spectrum}).
\end{restatable}

\textbf{Optimization {} {} } We perform gradient descent with fixed learning rate $\eta$,
\begin{equation}
    \label{eqn:gd}
    \W(t+1) = \W(t) - \eta \nabla\risk(\W(t)) = \W(t) - \eta\sum_i\loss'(\W^\transpose\y_i\x_i)\cdot \y_i\x_i.
\end{equation}

\textbf{Solution dynamics}{} {} For sufficiently small learning rate $\eta$, and (bounded) starting point~$\W(0)$, \citet{soudry2018implicit} showed that: 
\begin{equation}
    \label{eqn:weight_growth}
    \W(t) = \hat{\W} \log t + \rho(t),
\end{equation}
where $\rho(t)$ is a bounded residual term, and $\hat{\W}$ is the solution to the hard margin SVM: 
\begin{equation}
    \hat{\W} = \argmin_{\W\in\mathbb{R}^d} \norm{\W}{2}^2 \;\; s.t. \; \W^\transpose\y_i\x_i \geq 1,
\end{equation}

\subsection{First-split Learning}
For stage 1, we consider that we train the model for a maximum of $T$ epochs (until we achieve 100\% accuracy on the first training dataset $\dataset_A$). This means that the learned weight vectors are close to, but have not converged to the max margin solution. The solution at the end of $t$ epochs is given by $\W_A(t)$. At sufficiently large $T$, we have:
\begin{align}
    \begin{split}
    \W_A(T) &= \hat{\W}_A \log T + \rho_A(T) \\
    \W_A(T)^\transpose \y_i\x_i &\geq 1 \quad \forall (\x_i, \y_i)\in\dataset_A
    \end{split}
\end{align}

\subsection{Second-split Forgetting}
We initialize the weights for second stage of training with $\W_A(T)$ from first training stage, and then train on $\dataset_B$. We provide the formal theorem statement and complete proofs in Appendix~\ref{app:theory}, but provide informal theorem statements and an intuitive proof sketch below:

\begin{theorem}[Asymptotic Forgetting (informal)]
\label{theorem:asymptotic-informal}
For sufficiently small learning rate, given datasets $\dataset_A, \dataset_B \sim \dist^n$. After training for $T'\to\infty$ epochs, the following hold with high probability:
\begin{enumerate}%
  \setlength{\itemsep}{1pt}
  \setlength{\parskip}{0pt}
  \setlength{\parsep}{0pt}
    \item Mislabeled and Rare examples from $\dataset_A$ are forgotten.
    \item Complex examples from  $\dataset_A$ are not forgotten.
\end{enumerate}
\end{theorem}

\emph{Proof Sketch. }
We use the result from \citet{soudry2018implicit} that for any bounded initialization, when trained on a separable data, the model converges to the same min-norm solution. As a result, we can ignore the impact of $\dataset_A$ at infinite time training. Then, we use generalization bounds from \citet{chatterji2021finite} to argue about the accuracy on mislabeled and complex examples. For the case of rare examples, we show that the probability of correct model prediction can be approximated by a Gaussian CDF with mean 0 and $\mathcal{O}(1/\sqrt{n})$ variance.

\begin{theorem}[Intermediate-Time Forgetting (informal)]
\label{theorem:intermediate-time-informal}
For sufficiently small learning rate, given two datasets $\dataset_A, \dataset_B \sim \dist^n$. For a model initialized with weights, $\W_B(0) = \W_A(T)$ and trained for $T^\prime$ = f(T) epochs, the following hold with high probability: 
\begin{enumerate}%
  \setlength{\itemsep}{1pt}
  \setlength{\parskip}{0pt}
  \setlength{\parsep}{0pt}
    \item Mislabeled examples from $\dataset_A$ are no longer incorrectly predicted.
    \item Rare examples from $\dataset_A$ are not forgotten.
\end{enumerate}
\end{theorem}

\emph{Proof Sketch. }
$\dataset_B$ contains examples from the same majority distributions as $\dataset_A$. 
The mislabeled example also belongs to one of these distributions, but has the opposite label. 
However, $\dataset_B$ does not have samples from rare groups found in $\dataset_A$.
Using representer theorem, we decompose the model updates
into a weighted sum of each training data point in $\dataset_B$. Then, we analyze the change in prediction on rare and mislabeled examples, which is a 
dot product of the weight update with $\x_m$ or $\x_r$. 
Per our assumptions, the the mean of each group $\vmu_g$ is orthogonal to the other. As a result, the rare example finds negligible coupling with any example in $\dataset_B$, and the variance of its prediction keeps increasing due to the noise term contributed in the model weights from each example in $\dataset_B$. On the contrary, the mislabeled examples have a strong coupling with all the examples in its group. Due to its incorrect label, the mean of its predictions moves towards the correct label, with variance increasing at a similar rate. 
The final step is to jointly analyze the rate of change of prediction of both the examples, and find an optimal time $T'$ when the prediction on the mislabeled example is flipped and the rare example still retains its prediction with high probability.

\section{Conclusion}
While many prior works investigate training time dynamics 
to characterize the hardness of examples,
we enrich this literature with a complementary lens
focused on the second-split forgetting time.
We demonstrate the potential of SSFT
to distinguish among rare, mislabeled, and complex examples; and also show the differences in the example properties captured by first-split and second-split metrics.

Our work opens new lines of inquiry in future work
that can utilize the separation of hard examples.
First, we expect state of art methods in label noise identification 
to benefit by augmenting our approach. 
Further, we believe our ablations showing 
that complex, noisy, and mislabeled samples 
may all be learned slowly
inspire future work that can unite different takes
on the memorization-generalization research---early 
learning, simplicity bias, and singleton memorization.

\section*{Acknowledgements}
We would like to thank Aakash Lahoti and Jeremy Cohen for their insightful comments on this work.
SG acknowledges Amazon Graduate Fellowship and JP Morgan PhD Fellowship for their support.
ZL acknowledges Amazon AI, Salesforce
Research, Facebook, UPMC, Abridge, the PwC Center, the Block Center, the Center for Machine
Learning and Health, and the CMU Software Engineering Institute (SEI) via Department of Defense
contract FA8702-15-D-0002, for their generous support.

\bibliography{paper}

\bibliographystyle{plainnat}

\clearpage
\appendix
\onecolumn
\section*{\center{Characterizing Datapoints via Second-Split Forgetting\\Supplementary Material}}

\section{Theoretical Results}
\label{app:theory}
\subsection{Preliminaries}
\label{app:theory-preliminaries}

Let $\W \in \mathbb{R}^{d}$ represent the weight vector of overparametrized linear model.
We analyze the learning and forgetting dynamics by minimizing the empirical risk: $\risk(\dataset;\W) = 
\sum_i \loss(\W^\transpose\y_i\x_i).$ We consider the exponential loss $\ell(z) = \exp(-z)$ for our analysis. For completeness, we rewrite the definitions and preliminaries from the main paper below.

\paragraph{Data Generating Process} We restate the data generating process as detailed in the synthetic experiment in Section~\ref{subsec:spectrum}. We consider data ($\x,\y$) sampled from a mixture of multiple distributions $\dist_g$, s.t. $\x\in\mathbb{R}^d$.
$\dist_g$ denotes the $g^{\text{th}}$ group and has a sampling frequency of $\pi_g$.
Each group $\dist_g$ is a distribution over $(\distX_g \times \{\y_g\})$, 
i.e., the true label for all the samples drawn from a given group is the same, and the examples in each group are non-overlapping. 
Each group is parametrized by a set of $k\ll d$ unique indices $\indices_g \subset [d]$ such that $\indices_i \cap \indices_j = \phi$ for $i\neq j$. The discriminative characteristic of each group is the vector $\direction_g$, such that, $[\direction_g]_i = 1$ if $i\in\indices_g$ else $0 \; \forall i\in[d]$. In the following discussion, we will refer to $\mu_g$ as the coordinate-wise mean vector for the group $g$, such that $\vmu_g = \mu_g\direction_g$.
Then for any sample $(\x,\y)\in\dataset$:
\begin{equation}
\label{eqn:data-model}
P(\x \in \distX_g) = \pi_g; \quad \x|\distX_g \sim \mathcal{N}(0, \sigma^2 \identity_d) +\vmu_g.
\end{equation}

\defrare*

\setcounter{footnote}{0}
\defmislabeled*

\defcomplex*

\update{The implication of the aforementioned characterization is that complex distributions have a lower signal-to-noise ratio (SNR) as compared to simple distributions. We assume the sample complexity required to estimate the distribution as a proxy for the complexity of the distribution. In this regard, having a low SNR increases the complexity.}

Recall that $\dataset_A$ and $\dataset_B$ denote the first and second training splits of our dataset.
In our theoretical framework, we only consider two majority distributions in the second split dataset ($\dataset_B$). Let us call them $\dist_1, \dist_2$. 
Therefore, both $\dataset_A$ and $\dataset_B$ contain $O(n)$ samples from $\dist_1,\dist_2$.
In the first split dataset ($\dataset_A$), we consider the presence of another distribution $\dist_r$ that constitutes the rare example (only one sample $(\x_r, \y_r)$). The mislabeled example (only one sample $(\x_m, \y_m)$) belongs to one of the majority distributions, and we will assume without loss of generality that this is from distribution $\dist_1$. To understand the population accuracy in the case of complex distributions, we will draw a simple analogy in the case when the majority distributions $\dist_1, \dist_2$ occur from complex distributions as defined below.
Based on \eqref{eqn:data-model}, without loss of generality, we will assume that dimensions $\{1\dots k\}, \{k+1\dots 2k\}, \{2k+1\dots 3k\}$ are the predictive dimensions for the majority group 1, 2 and that for the rare example from dataset 1.
We make these assumptions to simplify the theoretical exposition. However, our results can be observed even after relaxing them at the expense of more book-keeping.

Based on \citet{chatterji2021finite}, we make the following assumptions about the problem setup: 

\textbf{(A.1) {} {} }The failure probability satisfies $0 \leq \delta \leq \nicefrac{1}{C}$,

\textbf{(A.2) {} {} }The number of samples satisfies $n \geq C \log\left(\nicefrac{1}{\delta}\right)$,

\textbf{(A.3) {} {} }The input dimension $d\geq C \max\{n^2\log(n/\delta),nk\mu^2/\sigma^2\}$,
and $\nicefrac{k\cdot \mu^2}{\sigma^2} \ge C \log\left(\nicefrac{n}{\delta}\right)$, 

where $\nicefrac{k\cdot \mu^2}{\sigma^2}$ represents the signal to noise ratio.

\paragraph{Optimization} We perform gradient descent with fixed learning rate $\eta$,
\begin{equation}
    \label{eqn:gd-restate}
    \W^{(t+1)} = \W(t) - \eta \nabla\risk(\W(t)) = \W(t) - \eta\sum_i\loss'(\W^\transpose\y_i\x_i)\cdot \y_i\x_i.
\end{equation}

\textbf{Asymptotic Solution}~\citep{soudry2018implicit} {} {} For sufficiently small learning rate $\eta$, and (bounded) starting point~$\W(0)$,
\begin{equation}
    \label{eqn:weight_growth-restate}
    \W(t) = \hat{\W} \log t + \rho(t),
\end{equation}
where $\hat{\W}$ is the solution to the hard margin SVM: 
\begin{equation}
    \hat{\W} = \argmin_{\W\in\mathbb{R}^d} \norm{\W}{2}^2 \;\; s.t. \; \W^\transpose\y_i\x_i \geq 1,
\end{equation}

\subsection{Learning Stage}
\label{app:theory-learning}
For the stage 1, we consider that we train the model for a maximum of $T$ epochs (until we achieve 100\% accuracy on the first training dataset $\dataset_A$ with margin greater than 1). This means that the learned weight vectors are close to, but have not converged to the max margin solution. The solution at the end of $T$ epochs is given by $\W_A(t)$. At sufficiently large $T$, we have:
\begin{align}
    \begin{split}
    \W_A(T) &= \hat{\W}_A \log T + \rho_A(T) \\
    \W_A(T)^\transpose \y_i\x_i &\geq 1 \quad \forall (\x_i, \y_i)\in\dataset_A
    \end{split}
\end{align}

\begin{lemma}[Bounded Weights]
\label{proposition:bounded-weights}
With probability at least $1-\delta$, there exists a bounded time t beyond which the model classifies all training points correctly.
\end{lemma}
\begin{proof}
From Lemma~\ref{lemma:separability}, we know that the dataset is separable with probability at least $1-\delta$. This means that the max-margin solution (or the SVM solution) for the dataset classifies all training points correctly.
From the analysis in \citet{soudry2018implicit}, we know that: 
\begin{align}
    \begin{split}
    \lim_{t\to\infty} \frac{\W_A(t)}{\norm{\W_A(t)}{2}}& =  \frac{\hat{\W}_A}{\norm{\hat{\W}_A}{2}},\\
    \lim_{t\to\infty}\y_i\x_i^\transpose\frac{\W_A(t)}{\norm{\W_A(t)}{2}} &=  \y_i\x_i^\transpose\frac{\hat{\W}_A}{\norm{\hat{\W}_A}{2}} 
    \end{split}
\end{align}

Then it directly follows that, for $\epsilon > 0$, there exists bounded time $T>0$ such that $\forall t > T$, 

$$\abs{\y_i\x_i^\transpose\frac{\W_A(t)}{\norm{\W_A(t)}{2}} -  \y_i\x_i^\transpose\frac{\hat{\W}_A}{\norm{\hat{\W}_A}{2}}} < \epsilon$$

This concludes that there exists a time T for which the sign of the prediction of both the max-margin solution and the learnt solution will be the same, implying correct prediction for every example in the training set for a bounded time solution.

\end{proof}

\subsection{Forgetting Stage}
\label{app:theory-forgetting}
We initialize the weights for second stage of training with $\W_A(T)$ from first training stage, and then train on $\dataset_B$ to minimize the empirical loss using gradient descent (Equation~\ref{eqn:gd}). Assume that the dataset is balanced in the class labels, $|\dataset_B| = n$. Also, recall that $\dataset_B$ does not contain mislabeled or rare examples from the same sub-group as in $\dataset_A$. For analyzing example forgetting in an asymptotic setting, we will directly borrow results from the analysis made by \citet{chatterji2021finite}. They prove a stronger result for the case where the dataset contains a fraction $\eta$ of mislabeled examples. However, we will use the setting without label noise. The asymptotic result then builds on to the main theorem of the paper on intermediate time forgetting (Theorem~\ref{theorem:intermediate-time-formal}).

\paragraph{Transformations for Equivalence to \citet{chatterji2021finite}}
In our setup, we consider a group structure where each distribution has a mean vector that is orthogonal to the others. We show the equivalence of the same to the data model studied in prior work~\cite{chatterji2021finite}.

Let us define $\vmu_1,\vmu_2\in\mathbb{R}^d$ as follows:
\begin{equation*}
\idx{\vmu_1}{j} = \begin{cases}
\mu  &\text{if $j\in \{1\dots k\}$}\\
0 &\text{o.w.}
\end{cases}
\end{equation*}
Similarly, define $\vmu_2, \vmu_3$ as well. For the equivalence condition, we are only concerned about the dataset split $\dataset_B$ which is where the generalization bounds hold.
Further, let $z\sim\mathcal{N}\left(0,\sigma^2\identity_d\right)$. Then, 
$\x|\distX_g \sim z +\vmu_g $.

We can now shift and rescale the axes such that the new origin is located at $\nicefrac{(\vmu_1 + \vmu_2)}{2}$. Then, define $\vmu = \nicefrac{(\vmu_1 - \vmu_2)}{2}$. This results in the simplification that the mean of the examples sampled from $\dist_1$ is $\vmu$ and that from $\dist_2$ is $-\vmu$. We can hence express $\x = \y\vmu + \z $. This directly follows their model assumption where $\z\in\mathbb{R}^d$ has each marginal sampled from a mean zero subgaussian distribution with subgaussian norm at most 1. In our case, each marginal is a Gaussian random variable with variance $\sigma^2$. Once again, we can rescale the axes such that $\Tilde{\vmu} = \nicefrac{\vmu}{\sigma}$. Now, our data model directly follows the data model discussed in prior work~\cite{chatterji2021finite}.

\subsection{Asymptotic Analysis}
\label{app:theory-asymptotic}

First, we analyze infinite-time training case. We will extend the result from \citet{soudry2018implicit, chatterji2021finite} to show that (i) mislabeled and rare examples are forgotten when the model is trained for long; and (ii) complex examples are not forgotten.
First, recall the result used in Subsection~\ref{app:theory-learning} for any bounded initialization for weights $\W_B(0)$,
\begin{align}
    \begin{split}
    \W_B(t) &= \hat{\W}_B \log t + \rho_B(t).
    \end{split}
\end{align}
We will first provide a formal version of Theorem~\ref{theorem:asymptotic-informal} which was informally stated in the main paper.
\begin{theorem}[Asymptotic Forgetting]
\label{theorem:asymptotic-formal}
Under assumptions A.1, A.2, A.3,
with probability 1-$\delta$, fine-tuning for $t\to\infty$ iterations on the second dataset $\dataset_B$ produces a max-margin classifier $\hat \W_B$ such that 
\begin{align*}
 \prob_{(\x_m, \y_m) \in \dataset_A} \left[ \sign(\hat \W_B \cdot \x_m) = \y_m  \right] &\leq  \exp\left(- \nicefrac{c\norm{\Tilde{\vmu}}{2}^2}{d}\right) \,, \\
 \Phi\left(-\nicefrac{1}{C}\right)\; \le \prob_{\;(\x_r,\;\y_r) \in \dataset_A} \left[ \sign(\hat \W_B \cdot \;\x_r) \neq \;\y_r  \right] &\le \Phi\left(\nicefrac{1}{C}\right) \,, \\
 \prob_{\;(\x_c,\;\y_c) \in \dataset_A} \left[ \sign(\hat \W_B \cdot \;\x_c) \neq \;\y_c  \right] &\leq \exp\left(- \nicefrac{c\norm{\Tilde{\vmu}/\lambda}{2}^2}{d}\right) \,,
\end{align*}
for some absolute constant c>0 and $(\x_m,\y_m)\in\mislabeled, (\x_r, \y_r) \in \atypical, (\x_c, \y_c) \in \complex$ respectively.
\end{theorem}

The theorem implies that the probability that mislabeled examples from $\dataset_A$ are classified with the given (incorrect) label tends to 0 if $\norm{\vmu}{2} = \theta(d^\beta)$ for any $\beta\in(1/4,1/2]$. Note that in our case, we have $k$ dimensions of signal. Therefore, as long as $k/d$ is a constant fraction, as the input dimensions $d\to\infty$ the above holds. Examples in $\dataset_A$ from distributions absent in $\dataset_B$ are randomly classified.

\emph{Intuition.} 
In the asymptotic case, the initial weights from first split training $\W_A(T) = \mathcal{O}(\log T) \ll \W_B(t)$, where $t\to\infty$ in the limit of infinite training. As a consequence, for any bounded initialization, the model weights converge to the minimum norm solution (SVM) solution from \citet{soudry2018implicit}.
We use results from \cite{chatterji2021finite} who study the setting of a binary classification problem with noisy label fraction $\eta$. In our case, since the second-split is assumed to have only clean samples, $\eta = 0$.  The final step is to adapt our data generating process to the format used in \citet{chatterji2021finite}.

\begin{proof}

Recall the transformations described in Section~\ref{app:theory-forgetting}. Let $(\x_m,\y_m), (\x_r, \y_r), (\x_c, \y_c)$ represent any point from $\dataset_A$ which belongs to mislabeled set $\mislabeled$, rare set $\atypical$ and complex set $\complex$. The important thing to note in the analysis that follows is that each of these examples is independent of the samples in $\dataset_B$. Hence, the probability of correctly predicting on them is same as that of correctly predicting on a population sample, in the limit of infinite training (when initialization does not matter and all models converge to the same solution). 

$$\prob_{(\x,\y) \sim \dist} \left[ \sign(\W \cdot \x) \neq \y \right] = \prob_{(\x,\y) \sim \dist} \left[ (\y\W \cdot \x) < 0\right]\,,$$

We will now analyze the probability of correct prediction of mislabeled, rare, and complex examples separately.

\paragraph{Mislabeled Examples}
Since $\dataset_B$ is separable, in the limit of infinite-training time, the classifier correctly predicts all the examples in the dataset. We will use this fact to show that it assigns the opposite label to the mislabeled sample in set $\dataset_A$ with high probability.
Then, we denote `failure' as the event that the mislabeled samples is still predicted with label $\y_m$ at infinite-time training. 
\begin{align*}
    \prob_{(\x,\y) \sim \dist_1} \left[ (\y\W \cdot \x) < 0\right]  
    &= 1 - \prob\left[ (\y_m\W \cdot \x_m) < 0\right]\\
    &=\prob\left[ (\y_m\W \cdot \x_m) > 0\right] \\
    &=\prob\left[ \sign(\W \cdot \x_m) = \y_m\right] \,,
\end{align*}
$$$$

Then, we can directly borrow the result from \citet{chatterji2021finite} (Theorem 4) to find that

\begin{align}
    \begin{split}
    \prob_{(\x,\y) \sim \dist_1} \left[ (\y\W \cdot \x) > 0\right]
    &\le \exp \left(-c\frac{||\Tilde{\vmu}||^4}{d}\right) \\
   \prob\left[ \sign(\W \cdot \x_m) = \y_m\right]
   &\le \exp \left(-c\frac{||\Tilde{\vmu}||^4}{d}\right) \,,
    \end{split}
\end{align}

\paragraph{Rare Examples}

Without loss of generality, we may assume that the correct label for the rare example $\y_r = 1$. 
\begin{align}
    \begin{split}
        \prob\left[ (\y_r\W \cdot \x_r) < 0\right] 
        &= \prob\left[ (\y_r\W \cdot (\x_r - \vmu_r) < -\y_r\W\cdot\vmu_r\right]\\
        &= \prob\left[ (\W \cdot (\x_r - \vmu_r) < -\W\cdot\vmu_r\right], \qquad \text{(since $\y_r$ = 1)}\\
        &= \prob \left[ (\W \cdot \z_r < -\W\cdot\vmu_r\right]\\
        &= \Phi\left(\frac{-\W\cdot\vmu_r}{\sigma\norm{\W}{2}}\right),
    \end{split}
\end{align}

where $\Phi$ is the Gaussian CDF. In the last step we use the fact that $\z_r \sim \mathcal{N}(0,\sigma^2 \identity_d)$. Therefore its dot product with the vector $\W$ results with a summation of $d$ Gaussian vectors, each with mean 0 and variance $\sigma^2 \idx{\W}{i}^2\; \forall i \in [d]$. Now what remains is to prove that the value at which we want to calculate the CDF is close to 0, so that the probability of the misclassification is close to 0.5.

Recall from the analysis in \citet{soudry2018implicit} that asymptotically the model converges to the max-margin separator that is comprised of a weighted sum of the support vectors of the dataset (let us call this set $\mathcal{V}_B$). This means that the final weights of the model $\W$ is a combination of datapoints from the first two majority distributions $\dist_1, \dist_2$. 
\begin{align}
    \W &= \sum_{i\in\mathcal{V}_B} \alpha_i \x_i, \\
   \vmu_r \cdot \W &= \sum_{i\in\mathcal{V}_B} \alpha_i (\vmu_r \cdot \x_i)
\end{align}
Since $\vmu_r \cdot \vmu_i = 0$, we have that $(\vmu_r \cdot \x_i) = \vmu_r\cdot\z_i + 0$, which is a mean zero random variable. 
\begin{align}
   \vmu_r \cdot \W 
   &= \sum_{i\in\mathcal{V}_B} \alpha_i (\vmu_r \cdot \z_i), \\ 
   &= \sum_{i\in\mathcal{V}_B} \alpha_i
   \mu\sum_{2k+1\leq j\leq3k} \idx{\z_i}{j}, \\
   &= \xi, \\
\end{align}

where $\xi\sim\mathcal{N}(0, k\mu^2\sigma^2\sum_{i\in\mathcal{V}_B}\alpha_i^2)$. 
Also, $\norm{\W}{2}^2 = \sum_{i\in\mathcal{V}_B}\alpha_i^2\x_i^2$. 
However, $\x_i = \vmu_i + \z_i$. 
From Lemma~\ref{lemma:gaussian-square}, we know that with probability greater than $1-\delta/6$, for every example, 
$\frac{d\sigma^2}{2} \le \norm{ \z_i}{2}^2 \le \frac{3d\sigma^2}{2}$. By Young's inequality for products:

\begin{align*}
\norm{\z_i}{2}^2 &= \norm{\x_i -\vmu_i}{2}^2,\\
&\leq 2\norm{\x_i}{2}^2 + 2\norm{\vmu_2}{2}^2,\\
\norm{\x_i}{2}^2
&\geq \frac{1}{2}\norm{\z_i}{2}^2 - \norm{\vmu_i}{2}^2,\\
&\geq \frac{d\sigma^2}{4} - \norm{\vmu_i}{2}^2, \\
&\geq \frac{d\sigma^2}{8} \qquad \qquad \text{(since $\frac{k\mu^2}{\sigma^2} < d/nC$ for large C)},\\
\norm{\W}{2}^2 = \sum_{i\in\mathcal{V}_B}\alpha_i^2\x_i^2
&\geq \frac{d\sigma^2}{8}\sum_{i\in\mathcal{V}_B}\alpha_i^2.
\end{align*}
Therefore, 
\begin{align}
\begin{split}
\eta &= \frac{-\xi}{\sigma\norm{\W}{2}}\sim\mathcal{N}\left(0, \frac{k\mu^2}{d\sigma^2/8}\right),\\
\eta &\sim \mathcal{N}\left(0, \frac{1}{nC}\right), \qquad \qquad \text{(since $\frac{k\mu^2}{\sigma^2} < d/nC$ for large C)},\\
 \prob\left[ (\y_r\W \cdot \x_r) < 0\right]
        &= \Phi\left(\eta\right)
\end{split}
\end{align}
Using Gaussian tail bound on $\eta$, with probability at least $1-\delta$, $\eta\le\sqrt{\frac{\log(1/\delta)}{nC}}$. Therefore, 
$\prob\left[ (\y_r\W \cdot \x_r) < 0\right]
        \leq \Phi\left(
        \sqrt{\frac{\log(1/\delta)}{nC}}
        \right)
        \leq \Phi\left(\frac{1}{C}\right)
        \approx 0.5$.

In the last step, we used the assumption (A.2) that $n\geq C \log(\nicefrac{1}{\delta})$ for some large constant C.

\emph{Remark:} This analysis highlights that the meaning of being `forgotten' for rare examples is to predict randomly. However, in case of mislabeled examples, the predicted label of the example approaches its `true' label, irrespective of its training label.

\paragraph{Complex Examples}
The analysis for complex examples follows directly from the analysis in the case of mislabeled examples. The only difference is the magnitude of the signal to noise ratio within the group. Recall that complex examples are also sampled from majority groups. 
Therefore, the probability that the complex example $(\x,\y)$ is misclassified at the end of training for infinite time is given by:

\begin{align}
    \begin{split}
    \prob \left[\y\x^\transpose \W < 0\right] 
    &\le \exp \left(-c\frac{||\nicefrac{\Tilde{\vmu}}{\lambda}||^4}{d}\right)
    \end{split}
\end{align}
This approaches 0, indicating perfect classification of complex examples from $\dataset_A$. Note that this is a complimentary case where the second split only has examples from the complex distribution.

\end{proof}

\subsection{Intermediate Time Analysis}

From the analysis in Section~\ref{app:theory-asymptotic}, we find that all the mislabeled and rare examples are forgotten by the time we train for $t\to\infty$ iterations. Since examples from complex subgroups are not forgotten even at infinite time training, we skip analysis for those examples in this subsection. 
Our goal is to show that there exists a time $T'$ such that  with high probability, the model forgets all the mislabeled examples, but still correctly classifies all the rare examples.
To show this, we will track the accuracy of inputs in the first data split, as we train on the examples in the second split.

The model output for any sample $(\x_i,\y_i)\in\dataset_A$ is given by $\W(t)^\transpose\x_i$, and the prediction is considered to be correct if $\sign(\W(t)^\transpose\x_i) = \y_i$, or if $\W(t)^\transpose\x_i\y_i > 0 $. From hereon, we will use the notation $\acc_i^t = \W(t)^\transpose\x_i\y_i$.

Recall that we considered that there is one example from both the mislabeled and rare example category in $\dataset_A$. We will denote these data points by $(\x_m,\y_m)\in\mislabeled$, and $(\x_r, \y_r)\in\atypical$ respectively.
Without loss of generality, we will assume that $(\x_m,\neg\y_m)$ was sampled from $\dist_1$ in the mixture of distributions $\dist$. All the examples in the second split $\dataset_B$ sampled from the same distribution are given by $\dataset_{B,1}\subset\dataset_B$. 
The remaining examples are in the set denoted by $\dataset_{B,\neg 1}\subset\dataset_B$.

From the learning time training dynamics, we know that the initialization of the weights for the second round of training is given by:
\begin{align}
    \begin{split}
    \W_B(0) &= \hat{\W}_A \log T + \rho_A(T). 
    \end{split}
\end{align}

Now, from the representer theorem, we can decompose the model weights at any iteration of second-split training as follows, 
\begin{align}
\label{eqn:representer-decomposition}
    \begin{split}
    \W_B(t) &= \hat{\W}_A \log T + \rho_A(T) + \sum_{j\in\dataset_B}\beta_j \y_j\x_j. 
    \end{split}
\end{align}

Note that we introduce an additional term $\y_j$ in the decomposition using representer theorem, but this can be done without loss of generality since $\y_j\in\{-1,+1\}$. This helps us in ensuring that each $\beta_j$ is non-negative as shown in Lemma~\ref{lemma:beta}.

From Lemma~\ref{lemma:bounded-time-prediction}, we know that there exists a bounded time $T'$ when the mislabeled example prediction flips.  We will denote $\mu$ as the coordinate-wise mean for the $k$-signal dimensions in the vector $\vmu_g$ in the discussion that follows.
Let us define $\Delta_t = \frac{\sum_{j\in\dataset_{B,1}}\beta_j(t)}{\sum_{j\in\dataset_{B}}\beta_j(t)}$, and $\Delta = \max_{t} \Delta_t$. 
For a symmetric distribution with two majority subgroups with the opposite label, we would expect this value to be close to 0.5.
Now, we present formal version of Theorem~\ref{theorem:intermediate-time-informal} from the main paper.

\begin{theorem}[Intermediate-Time Forgetting]
 \label{theorem:intermediate-time-formal}
 Under 
 the distribution outlined in Appendix~\ref{app:theory-preliminaries} with
 assumptions A.1, A.2, A.3, whenever $\dataset_A$ is separable, when fine-tuning on the second dataset split $\dataset_B$ with sufficiently small learning rate, there exists some bounded time T' when
 \begin{enumerate}
  \setlength{\itemsep}{1pt}
  \setlength{\parskip}{0pt}
  \setlength{\parsep}{0pt}
  \item $\prob_{(\x_m,\y_m)\in\dataset_A}[\y_m \neq \W(T')\cdot\x_m ] \geq 1 - c_0\exp{\left(\frac{-k^2 \mu^2 \Delta^2}{c \sigma^2 d}\right)}-c_1\exp(-cd)$%
     \item $\prob_{(\x_r,\y_r)\in\dataset_A}\;\;[\y_r \;= \W(T')\cdot\x_r ]\; \geq 1 - c_0\exp{\left(\frac{-k^2 \mu^2 \Delta^2}{c \sigma^2 d}\right)}-c_1\exp(-cd)$ %
 \end{enumerate}
 
for absolute constants $c_0, c_1, c$ and $(\x_m,\y_m)\in\mislabeled, (\x_r, \y_r) \in \atypical$ respectively.
 \end{theorem}
 
If the fraction of dimensions that contain the signal, $\nicefrac{k}{d}$, is considered fixed then both the first and second term above exponentially decay with a factor of $d$. This suggests that increasing overparametrization leads to a higher likelihood of the phenomenon of intermediate time forgetting---there exists an epoch when the prediction of mislabeled example is flipped but the rare examples are still correctly predicted with high probability.

In what follows, the key idea is to show that at a time when the prediction of the mislabeled example is incorrect with high probability, the predictions of the rare examples is correct with high probability.

\begin{proof}
From Lemma~\ref{lemma:support-vectors-sa}, we know that $(\x_m,\y_m)$ and $(\x_r,\y_r)$ are support vectors for $\dataset_A$. Hence, $\y_m\x_m^\transpose\hat{\W}_A = \y_r\x_r^\transpose\hat{\W}_A  = 1$.
We will use this fact
to find the distribution for $\acc_m^t,\; \acc_r^t$.

Let $\idx{\x_j}{i}$ denote the $i^{th}$ dimension of the $j^{th}$ datapoint in the second split.
Recall that the second split comprises of two majority subgroups. $k$ dimensions of the input vector contain the true signal for class prediction. The dimensions $\{1\dots k\}, \{k+1\dots 2k\}, \{2k+1\dots 3k\}$ are the predictive dimensions for the majority group 1, 2 and that for the rare example from dataset $\dataset_A$. 
Therefore, $\idx{\x_j}{i} = \mu + \idx{\z_j}{i}$, if $i$ is in the predictive dimensions, otherwise $\idx{\x_j}{i} = \idx{\z_j}{i}$  where $\idx{\z_j}{i}\sim \mathcal{N}(0,\sigma^2)$. 
To make notation simple, we will refer to $\beta_j(t)$ by using $\beta_j$.

\emph{Remark:} We do not perform input transformation to prove the following results.

\paragraph{Mislabeled Example} The prediction on the mislabeled point times the given label can be written as:
\begin{align}
    \begin{split}
    \acc_m^t = \y_m\x_m^{\transpose}\W_B(t) &= \y_m\x_m^{\transpose}\left(\hat{\W}_A \log T + \rho_A(T) + \sum_{j\in\dataset_B}\beta_j \y_j\x_j\right), \\
    &= \log T + c_m + \sum_{j\in\dataset_B}\beta_j (\y_m\x_m^{\transpose}\y_j\x_j)\, \qquad \text{(using Lemma~\ref{lemma:support-vectors-sa})}, 
    \end{split}
\end{align}
where $c_m =\y_m\x_m^T\rho_A(T)$ is a residual term that does not continue grow with $T$ (Theorem 9 \citep{soudry2018implicit}). 

Without loss of generality, assume that the mislabeled example $\x_m$ belongs to majority group 1 (true label = 1), and was originally labeled such that $\y_m$ = -1 in the first-split dataset.

\begin{align}
    \begin{split}
    \sum_{j\in\dataset_B}\beta_j (\y_m\x_m^{\transpose}\y_j\x_j) &= \sum_{j\in\dataset_{B,1}}\beta_j (\y_m\x_m^{\transpose}\y_j\x_j) 
    + \sum_{j\in\dataset_{B,\neg1}}\beta_j (\y_m\x_m^{\transpose}\y_j\x_j) \\
    &= - \sum_{j\in\dataset_{B,1}}\beta_j (\x_m^{\transpose}\x_j) 
    + \sum_{j\in\dataset_{B,\neg1}}\beta_j (\x_m^{\transpose}\x_j).
    \end{split}
\end{align}

Now we can use the distribution properties of the dataset to further simplify per dimension and aggregate the sum across all examples.
\begin{align}
    \begin{split}
    \sum_{j\in\dataset_{B,1}}\beta_j (\x_m^{\transpose}\x_j)  &= 
    \x_m^{\transpose} \sum_{j\in\dataset_{B,1}}\beta_j \x_j = \x_m^{\transpose} \x_{\dataset_{B,1}},
    \end{split}
\end{align}

where $\idx{\x_{\dataset_{B,1}}}{i} =  \mu\; \mathbb{1}(i \in \{1\dots k\})\sum_{\dataset_{B,1}}\beta_j\ + \idx{\z_{\dataset_{B,1}}}{i}$, where $\idx{\z_{\dataset_{B,1}}}{i}\sim \mathcal{N}(0, \sigma^2 \sum_{j\in\dataset_{B,1}}\beta_j^2)$. Now, we will add up the dot product dimension wise. Let us call 
$B_1 = \sum_{\dataset_{B,1}}\beta_j $ and $B_1^v = \sum_{j\in\dataset_{B,1}}\beta_j^2 $. Then $\idx{\x_{\dataset_{B,1}}}{i} \sim \mathcal{N} (\mu\; \mathbb{1}(i \leq k\})\cdot B_1, \sigma^2 B_1^v)$.

\begin{align}
    \begin{split}
    \x_m^{\transpose} \x_{\dataset_{B,1}} &= \mu\sum_k \idx{\z_{\dataset_{B,1}}}{i} + \mu B_1 \sum_k \idx{\z_{m}}{i} + k\cdot \mu^2 \cdot B_1 + \sum_d \idx{\z_{\dataset_{B,1}}}{i} \cdot \idx{\z_m}{i} \\
    &= \mu \cdot \alpha_1 + \sum_d \idx{\z_{\dataset_{B,1}}}{i} \cdot \idx{\z_m}{i} \\
    \x_m^{\transpose} \x_{\dataset_{B,\neg 1}} &= \mu \cdot \alpha_2 + \sum_d \idx{\z_{\dataset_{B,\neg 1}}}{i} \cdot \idx{\z_m}{i} 
    \end{split}
\end{align}

where $\alpha_1~\sim \mathcal{N}(k\cdot \mu \cdot B_1, k\cdot\sigma^2 (B_1^v + B_1^2) )$ and $\alpha_2~\sim \mathcal{N}(0, k\cdot\sigma^2 (B_2^v + B_2^2) )$.
Notice that in the first step, the first three terms are independent of each other and can be added directly to obtain a new random variable using the independence condition, however, the last term is dependent on the first two. In the final step, we also add the terms for the dot product corresponding to the opposite group.

This gives us the overall expression as follows:
\begin{align}
\begin{split}
\sum_{j\in\dataset_B}\beta_j (\y_m\x_m^{\transpose}\y_j\x_j) &=\mu \cdot (\alpha_2 - \alpha_1) + \sum_d (\idx{\z_{\dataset_{B,\neg 1}}}{i} - \idx{\z_{\dataset_{B, 1}}}{i}) \cdot \idx{\z_m}{i}, \\
&= \mu \cdot \alpha_m  + \sum_d \idx{\z_{\dataset_{B}}}{i}\cdot \idx{\z_m}{i} - k\cdot \mu^2 \cdot B_1,
\end{split}
\end{align}
where $\alpha_m~\sim \mathcal{N}(0, k\cdot\sigma^2 (B^v + B_1^2 + B_2^2))$ and 
$\idx{\z_{\dataset_{B}}}{i}\sim \mathcal{N}(0, \sigma^2 B^v)$, since ($B_1^v + B_2^v = \sum_{j\in\dataset_{B}}\beta_j^2 = B^v$)
.
\begin{align}
\begin{split}
    \acc_m^t  = \log T + c_m + \xi_m - k\cdot \mu^2 \cdot B_1;  \text{ s.t. } \xi_m = \mu\cdot \alpha_m  + \sum_d \idx{\z_{\dataset_{B}}}{i}\cdot \idx{\z_m}{j}.
\end{split}
\end{align}

\paragraph{Rare Example}
Following the analysis of the mislabeled example, we can similarly find an expression for $\acc_r^t$ on the rare example as follows:
\begin{align}
    \begin{split}
    \acc_r^t = \y_r\x_r^{\transpose}\W_B(t) &= \y_r\x_r^{\transpose}(\hat{\W}_A \log T + \rho_A(T) + \sum_{j\in\dataset_B}\beta_j \y_j\x_j), \\
    &= \log T + c_r + \sum_{j\in\dataset_B}\beta_j (\y_r\x_r^{\transpose}\y_j\x_j) \qquad \text{(using Lemma~\ref{lemma:support-vectors-sa})},
    \end{split}
\end{align}

Following the same procedure as for mislabeled example, we get the overall expression as follows:
\begin{align}
\begin{split}
\sum_{j\in\dataset_B}\beta_j (\y_r\x_r^{\transpose}\y_j\x_j) &= \mu\cdot \alpha_r  + \sum_d \idx{\z_{\dataset_{B}}}{i}\cdot \idx{\z_r}{i},
\end{split}
\end{align}

where $\alpha_r~\sim \mathcal{N}(0, k\cdot\sigma^2 (B^v + B_1^2 + B_2^2))$ and 
$\idx{\z_{\dataset_{B}}}{i}\sim \mathcal{N}(0, \sigma^2 B^v)$.

\begin{align}
\begin{split}
\acc_r^t  = \log T + c_r + \xi_r, \text{ s.t. } \xi_r = \mu\cdot \alpha_r  + \sum_d \idx{\z_{\dataset_{B}}}{i}\cdot \idx{\z_r}{i}.
\end{split}
\end{align}

\paragraph{Combining both cases}
Recall that $\beta_j>0$ for all $j$. Therefore, $B^2 = (\sum_{j\in\dataset_{B}}\beta_j)^2 > B^v =\sum_{j\in\dataset_{B}}\beta_j^2 $.
Based on the analysis of \citet{soudry2018implicit}, we know that $\W_B(t)$ grows as fast as $O(\log t)$. Therefore, both $B_1 = O(\log t)$, and $B = O(\log t)$ (must grow at most as fast as that). 
From the problem definition, $B_1 \leq \Delta B$, where $\Delta \in \left[0,1\right]$.

Our goal now is to find an epoch $t$ during the second-split training when the rare example is correctly classified with high probability, and the mislabeled example is incorrectly classified with high probability. The maxima is achieved when $\acc_r^t \approx - \acc_m^t$. We assume that the step sizes are sufficiently small such that we can achieve this condition. Moreover, $c_r, c_m \ll \log T$. Hence, the condition is to find $t$ such that $\xi_r + \log T > 0$ and  $\xi_m + \log T < k \cdot \mu^2 \cdot B_1$ with high probability. By symmetry, we must find $\prob\left(|\xi_m| < k \cdot \mu^2 \cdot B_1 /2\right)$.

\begin{align}
    \begin{split}
        \prob\left(|\xi_m| > k \cdot \mu^2 \cdot \frac{B_1}{2}\right) \leq \prob\left(|\xi_1| > k \cdot \mu^2 \cdot \frac{B_1}{4}\right) + \prob\left(|\xi_2| > k \cdot \mu^2 \cdot \frac{B_1}{4}\right), 
    \end{split}
\end{align}

where $\xi_1 = \mu\cdot \alpha_m \sim \mathcal{N}(0, \mu^2 \sigma^2 k(B^v + B_1^2 + B_2^2)$ and $\xi_2 = \sum_d \idx{\z_{\dataset_{B}}}{i}\cdot \idx{\z_r}{i}$. For the first term, we can directly use the Gaussian tail bound using Chernoff method. 
\begin{align}
    \begin{split}
         \prob\left(|\xi_1| > k\cdot \mu^2 \cdot \frac{B_1}{4}\right) 
         &\leq 
         2 \exp{\frac{-k^2 \mu^4 B_1^2}{32 \mu^2 \sigma^2 d (B^v + B_1^2 + B_2^2)}},\\
         &\leq 2 \exp{\frac{-k^2 \mu^4 B_1^2}{32 \mu^2 \sigma^2 d (2B^2)}},\\
         &\leq 2 \exp{\frac{-k \mu^2\Delta^2}{c_1 \sigma^2}}.
    \end{split}
\end{align}

Now, for the second term, 
using Lemma~\ref{lemma:gaussian-product} we have that. 

\begin{align}
    \begin{split}
         \prob\left(|\xi_2| > k \cdot \mu^2 \cdot \frac{B_1}{4}\right) 
         \leq 
         2 \exp{\frac{-k^2 \mu^4 B_1^2}{c_2 \sigma^4 d B^v}} +  c_1\exp(-cd) ,\\
         \leq 2\exp{\frac{-k^2 \mu^2 \Delta}{c_2 \sigma^2 d}} +  c_1\exp(-cd),
    \end{split}
\end{align}

since $\frac{\mu}{\sigma} > 1$. Finally, combining both the cases we have that 
\begin{align}
    \begin{split}
        \prob\left(|\xi_m| > k \cdot \mu^2 \cdot \frac{B_1}{2}\right) \leq c_0\exp{\frac{-k^2 \mu^2 \Delta^2}{c \sigma^2 d}} + c_1\exp(-cd), 
    \end{split}
\end{align}

This shows that as the dimensionality of the dataset increases, the likelihood of prediction on mislabeled examples being flipped while the rare examples retain their prediction increases exponentially.
This concludes the proof of Theorem~\ref{theorem:intermediate-time-formal}.
\end{proof}

\subsection{Concentration Inequalities and Additional Lemmas}
To make our work self-contained, we supplement the reader with additional Lemmas and Theorems that are helpful tools for proving the theorems in this work. We restate versions of the Hoeffding and Bernstein inequalities as in \citet{chatterji2021finite}.

\begin{lemma}[\citet{soudry2018implicit}, Theorem 3]
\label{lemma:souddry}
For any linearly separable dataset $\dataset_A$ and for all small enough step-sizes $\alpha$, 
we have
\[
\frac{\W_A}{\| \W_A \|} = \lim_{t \rightarrow \infty} \frac{\W^{(t)}}{\| \W^{(t)} \|}.
\]
\end{lemma}

\begin{theorem}[General Hoeffding's Inequality]
\label{theorem:hoeffding-general}
Let $\theta_1, \dots, \theta_m$ be independent mean-zero sub-Gaussian random variables and $a = (a_1, \dots, a_m)\in 
\mathbb{R}^m$. Then, for every $t>0$, we have
$$\prob\left[\left\vert\sum_{i=1}^{m} a_i\theta_i\right\vert \geq t\right] \leq 2\exp{\left(\frac{-c_0\;t^2}{K^2\norm{a}{2}^2} \right)}, $$
where $K=\max_i \norm{\theta_i}{\psi_2}$ and c is an absolute constant.
\end{theorem}

In our case, since we deal with Gaussian random variables, $\norm{\theta}{\psi_2}$ (sub-Gaussian norm) is same as the variance of the random variable. That is, $\theta\sim\mathcal{N}(0,\sigma^2) \implies \norm{\theta}{\psi_2} = K = \sigma$. 

\begin{theorem}[Bernstein Inequality]
\label{theorem:bernstein}
For independent mean-zero sub-exponential random variables $\theta_1,\ldots,\theta_m$, for every $t>0$, we have
\begin{align*}
\prob\left[\Big\lvert \sum_{i=1}^m  \theta_i \Big\rvert \ge t\right]\le 2\exp\left(-c_1 \min\left\{\frac{t^2}{\sum_{i=1}^m \norm{\theta_i}{\psi_1}^2},\frac{t}{\max_i \norm{\theta_i }{\psi_1}}\right\}\right),
\end{align*}
where $c_1$ is
an absolute constant.
\end{theorem}

In our case, let $\x_i\sim\mathcal{N}(0,\sigma^2)$ be independent Gaussian random variables, then for $Z = \sum_i^d X_i^2$, 
$$ \prob(\left(\lvert Z - d\sigma^2\rvert\geq t\right) \leq 2\exp\left(-\frac{c_1}{\sigma^2}\min\left\{\frac{t^2}{d\sigma^2},t\right\}\right).$$

\begin{lemma}[Gaussian Product]
\label{lemma:gaussian-product}
Let $\z_1\sim\mathcal{N}(0,\sigma_1^2\identity_d), \z_2\sim\mathcal{N}(0,\sigma_2^2\identity_d)$ be independent multivariate Gaussian random variables. Then, 
$$\prob\left[\lvert \z_1 \cdot \z_2 \rvert > t \right] \leq 2\exp{\left(\frac{-c\;t^2}{d\sigma_1^2\sigma_2^2} \right)} + 2\exp{\left(-c_0 d\right)}$$ 
\end{lemma}

\begin{proof}
The proof of this lemma is a simplified version of the proof for Lemma 20 in \citet{chatterji2021finite}.
First, consider $\z_2$ to be fixed, and only $\z_1$ to be random. Then from Theorem~\ref{theorem:hoeffding-general} we have
$$\prob\left[\left\vert\sum_{i=1}^{d} \idx{\z_2}{i} \cdot \idx{\z_1}{i}\right\vert \geq t\right] \leq 2\exp{\left(\frac{-c\;t^2}{\sigma_1^2\norm{\z_2}{2}^2} \right)}.$$

Also, adapting Theorem~\ref{theorem:bernstein} such that $Z = \sum_i^d X_i^2 = \norm{\z_2}{2}^2$, and setting $t = d\sigma^2$
\begin{align*}
\prob(\left(\lvert\norm{\z_2}{2}^2 - d\sigma_2^2\rvert\geq t\right) &\leq 2\exp\left(-\frac{c_0}{\sigma^2}\min\left\{\frac{t^2}{d\sigma_2^2},t\right\}\right),\\
\prob(\left(\norm{\z_2}{2}^2 \geq 2d\sigma_2^2\right) &\leq 2\exp\left(-c_0 d\right).    
\end{align*}

Coming back to the initial problem and again considering both $\z_1, \z_2$ to be random variables, we have
\begin{align}
\label{e:separate}
\prob\left[\lvert \z_1 \cdot \z_2 \rvert \ge t\right] 
&\leq \prob\big[\lvert \z_1 \cdot \z_2 \rvert \ge t \;\big|\; 
      \norm{\z_2 }{2}^2 \leq  2d\sigma_2^2 \big]
      + \prob\left[ \norm{\z_2 }{2}^2 >  2d\sigma_2^2 \right], \\
      & \leq 2\exp{\left(\frac{-c\;t^2}{d\sigma_1^2\sigma_2^2} \right)} + 2\exp{\left(-c_0 d\right)}.
\end{align}

\end{proof}

\begin{corollary}[\citet{chatterji2021finite}, Lemma 20]
\label{cor:gaussian-product}
There is a $c \geq 1$ such that,
for all large
enough $C$,
with probability at least $1-\delta/6$, for all $i\neq j \in [n]$,
\begin{align*}
\lvert \z_i \cdot \z_j\rvert < c\left(\sigma^2\sqrt{d \log(n/\delta)}\right).\label{event:2}\\
\end{align*}
\end{corollary}

\begin{lemma}[Gaussian Square]
\label{lemma:gaussian-square}
There is a $c \geq 1$ such that, for all large enough $C$, with probability at least $1-\delta/6$, for all $k \in [n]$,$$\frac{d\sigma^2}{2} \le \norm{\z_k}{2}^2 \le \frac{3d\sigma^2}{2}.$$
\end{lemma}
\begin{proof}
Recall that $\x_k = \vmu_k + \z_k$, where $\z_k\sim\mathcal{N}(0,\sigma^2\identity_d)$. Then, 
$$\norm{\z_k}{2}^2 = \sum_i^d \idx{\z_k}{i}^2 = Z$$

Adapting Theorem~\ref{theorem:bernstein}, and setting $t = \lambda d\sigma^2$ with $0 < \lambda<1$
\begin{align*}
\prob\left(\lvert\norm{\z_2}{2}^2 - d\sigma^2\rvert\geq t\right) &\leq 2\exp\left(-\frac{c_1}{\sigma^2}\min\left\{\frac{t^2}{d\sigma^2},t\right\}\right),\\
\prob\left(\lvert\norm{\z_2}{2}^2 - d\sigma^2\rvert\geq \lambda d\sigma^2\right)
&= 2\exp\left(-\frac{c_1}{\sigma^2}\min\left\{\lambda^2 d\sigma^2,\lambda d\sigma^2\right\}\right), \\
&= 2\exp\left(-c_1 \lambda^2 d\right). \qquad \text{(since $0 < \lambda < 1$)}   
\end{align*}

Recall that $d\geq C\log (n/\delta)$. We can set $\lambda = 1/2$ so that, $\norm{\z_2}{2}^2 > d\sigma^2/2$ with probability at least $ 1 - \delta/6n$ (for large enough $C$). Taking a union bound over all examples gives us the desired result.

Note that we can get closer to $d\sigma^2$ by chosing an appropriately higher value of $C$.

\end{proof}

\begin{lemma}[Dataset Separability]
\label{lemma:separability}
With probability at least $1-\delta$ over random samples of dataset $\dataset_A$, samples $(\x_1,\y_1),\dots,(\x_n,\y_n)$ are linearly separable.
\end{lemma}

\begin{proof}
We will show that there exists a set of weights that with high probability correctly classify each example in the dataset. 
Let $\x_i = \vmu_i\ + \z_i$ for every data point in $\dataset$. 
From assumptions, our dataset ($\dataset_A$) contains $\z_r,\z_m$ that belong to rare and mislabeled groups. Rest all points are denoted by $\z_k$.
Consider the classifier $\W = \sum_i^n \y_i\x_i$. Then,

\paragraph{Case 1: Rare Example $(\x_r, \y_r)$}
\begin{align*}
\y_r \W \cdot \x_r 
   & = \sum_j \y_j \y_r \x_j \cdot \x_r,
   \\
   &= \x_r\cdot\x_r + \sum_{i \neq j} \y_r \y_j \x_r \cdot \x_j,\\
   &= \vmu_r\cdot\vmu_r + \z_i\cdot\z_i + \sum_{r \neq j} \y_r \y_j \z_r \cdot \z_j, \qquad \text{(since $\vmu_r\cdot\vmu_i=0\;\forall i\in\{1,2\}$)}\\ 
   &= k\mu^2 + \z_i\cdot\z_i + \sum_{r \neq j} \y_r \y_j \z_r \cdot \z_j, \\ 
   & \geq 0 + d\sigma^2/2 - c_1 n \sqrt{d\sigma^2 \log(n/\delta)} \qquad \text{(using Lemma~\ref{lemma:gaussian-square} and Corollary~\ref{cor:gaussian-product})}
   \\
   & > 0
\end{align*}
for $d\geq C n^2\log(n/\delta).$
\paragraph{Case 2: Mislabeled Example $(\x_m, \y_m)$}
Without loss of generality, assume that the mislabeled example is sampled from the distribution $\dist_1$, and the set of all correctly labeled examples in the dataset $\dataset_A$ from this distribution be $\dataset_{A,1}$.
\begin{align*}
\y_m \W \cdot \x_m 
   & = \sum_j \y_j \y_m \x_j \cdot \x_m,
   \\
   &= \x_m\cdot\x_m + \sum_{i \neq j} \y_m \y_j \x_m \cdot \x_j,\\
   &= \vmu_m\cdot\vmu_m + \z_i\cdot\z_i -\sum_{j\in\dataset_{A,1}}\vmu_m\cdot\vmu_m +  \sum_{m \neq j} \y_m \y_j \z_m \cdot \z_j, \\
   & \qquad \qquad \qquad \qquad  \qquad \qquad \qquad \qquad \text{(since $\vmu_m\cdot\vmu_1=1, \vmu_m\cdot\vmu_2=0, \vmu_m\cdot\vmu_r=0$)}\\ 
   & \geq 0 + d\sigma^2/2 - n_1 k\mu^2 - c_1 n \sqrt{d \log(n/\delta)}, \\
    & \qquad \qquad \qquad \qquad  \qquad \qquad \qquad \qquad \text{(if $n_1 = |\dataset_{A,1}|$, using Lemma~\ref{lemma:gaussian-square} and Corollary~\ref{cor:gaussian-product})} \\
   & \geq d\sigma^2/2 - n k\mu^2 - c_1 n \sigma^2\sqrt{d \log(n/\delta)}, \\
   & > 0
\end{align*}
for $d\geq C \max\{n^2\log(n/\delta),nk\mu^2/\sigma^2\}.$
\paragraph{Case 3: Majority Example $(\x_i, \y_i)$}
The case of majority examples directly follows from the case of mislabeled examples. Rather than having a negative summation over the mean vector for n examples (in line 3), we will have a positive summation because the true label will match the label of the rest of the examples in the subset $\dataset_{A,1}$. This will make the expected value of the dot product even larger.

\emph{Remark:} Since the first split dataset
$\dataset_A$ is separable, 
it directly follows that the second split dataset
$\dataset_B$ 
is also separable since it does not have any mislabeled and rare examples.

\end{proof}

\subsection{Lemmas for Theorem~\ref{theorem:intermediate-time-formal}}
\begin{lemma}[Sign of $\beta$]
\label{lemma:beta}
$\beta_j \geq 0$ for all j in Equation~\ref{eqn:representer-decomposition}.
\end{lemma}
\begin{proof}
Analyzing the steps of gradient descent, we have:
\begin{align}
    \begin{split}
    \dot{\W}(t) &= -\nabla\risk(\W(t)) \\
    &= \sum_{j\in\dataset_B} \exp{(-\y_j\x_j^\transpose\W(t))} (\y_j\x_j^\transpose) \\
    \W(t) - \W(0) &= \sum_{j\in\dataset_B}\left((\y_j\x_j^\transpose) \underbrace{\int_0^t\exp{(-\y_j\x_j^\transpose\W(t))}dt}_{\beta_j(t) } \right), \\
    \end{split}
\end{align}

Hence, $\beta_j \geq 0 \; \forall \; j$.
\end{proof}

\begin{lemma}[Support Vectors]
\label{lemma:support-vectors-sa}
If dataset $\dataset_A$ is separable, then
$(\x_m,\y_m)$ and $(\x_r, \y_r)$ are support vectors for $\dataset_A$.
\end{lemma}
\begin{proof} 
We will prove by contradiction. From our assumption, $(\x_m,\y_m)$ and $(\x_r,\y_r)$ are the only mislabeled and rare examples in the first-split $\dataset_A$ from their respective sub-groups. 
Let us assume that they are not support vectors. Then, we can directly follow from the Asymptotic Analysis in Subsection~\ref{app:theory-asymptotic} that the probability of correct classification of the rare example is 0.5 and for the mislabeled example approaches 0 as the model is trained for infinite time. But we know that the model achieves 100\% accuracy on the training set $\dataset_A$ with weights $\W_A(T)$. Hence, this is a contradiction, and $(\x_m,\y_m)$, $(\x_r,\y_r)$ must be support vectors for the original classification problem.
\end{proof}

\begin{lemma}[Bounded Time Prediction]
\label{lemma:bounded-time-prediction}
With probability at least $1-\delta$, there exists a bounded time T' = O(T) when the mislabeled examples are incorrectly classified with high probability.
\end{lemma}
\begin{proof}
In Subsection~\ref{app:theory-asymptotic} we have shown that at infinite time training, the model misclassifies mislabeled examples with high probability.
Then, the existence of bounded time weights for which the prediction of mislabeled examples is flipped directly follows from the proof of Proposition~\ref{proposition:bounded-weights} applied to the result of Theorem~\ref{theorem:asymptotic-formal}.
\end{proof}

\section{Experimental Results}
\label{app:empirical}
\subsection{Experimental Setup}
\label{app:exp-setup}

\paragraph{Architectures} We perform experiments using four different model architectures---LeNet, ResNet-9~\citep{resnet-9}, ResNet-50, and Bert-base-cased~\citep{devlin2018bert}. Comparisons with model architectures are used in analysis of stability of the SSFT metric. For other numbers reported in tables and plots, we use the ResNet-9 model, unless otherwise stated.

\paragraph{Optimizer}
We experiment with three different learning rate scheduling strategies---cyclic learning rate schedule, cosine learning rate, and step decay learning rate. We test for two values of peak learning rate---0.1 and 0.01. All the model are trained using the SGD optimizer with weight decay 5e-4 and momentum 0.9, apart from the comparison with optimizers in Appendix~\ref{app:stability} where we also experiment with the Adam optimizer.

\paragraph{Training Procedure} We train for a maximum of 100 epochs or until we have 5 epochs of 100\% training accuracy. We first train on $\dataset_A$, and then using the pre-initialized weights from stage 1, train on $\dataset_B$ with the same learning parameters. All experiments can be performed on a single RTX2080 Ti. 

\begin{figure*}[t]
\centering
  \includegraphics[width=0.7\linewidth]{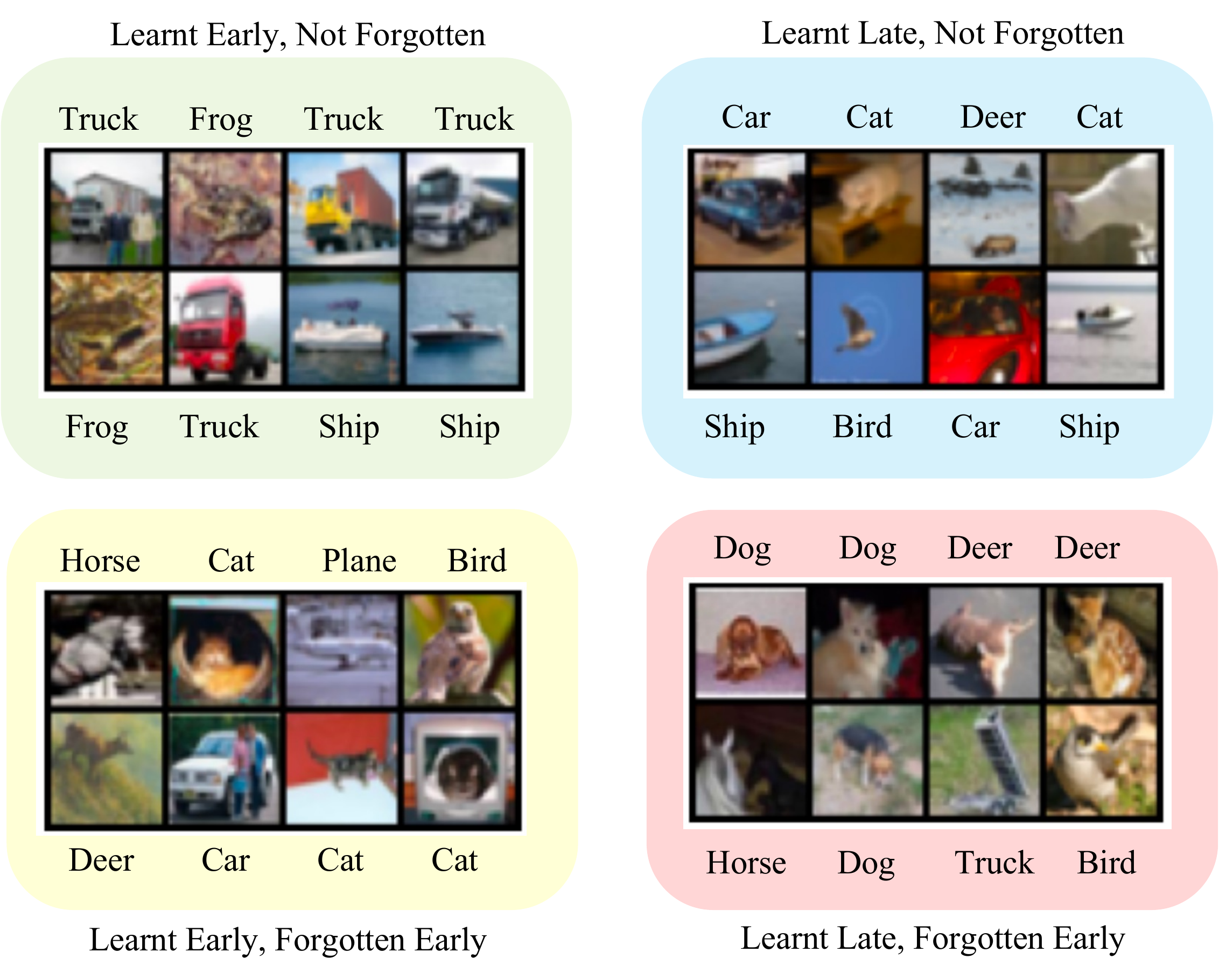}
\caption{Examples from the CIFAR-10 dataset grouped based on their learning and forgetting time.}
 \label{fig:cifar-10-spectrum}
\end{figure*}

\subsection{Image Datasets}
\label{app:image-graphs}
In the main paper we present visualizations of training examples from the MNIST dataset based on which quadrant they lie on in the learning-forgetting graph. Here, we complement our findings by showing visualizations for the CIFAR-10 dataset.
We note that CIFAR-10 dataset provide many different types of visibility patterns within the same class. Hence, examples may be learnt late due to belonging to a rare visibility pattern. In Figure~\ref{fig:cifar-10-spectrum}, we see that the examples that were learnt earliest and never forgotten have similar visibility patterns---for instance all the trucks have a similar perspective. On the contrary, as we move to the first quadrant with examples that were learnt late but never forgotten, we see that all the examples are true to their semantic class, but these visibility patterns occur rarely. 
Finally, we also analyze the visualizations based on examples that were forgotten during the course of second-split training. While in the case of MNIST dataset, SSFT was able to remove the mislabeled examples well, we see that CIFAR-10 offers more challenges because examples may be ambiguous because of other reasons and may be forgotten owing to the model using spurious features.

\section{Ablation Studies}
\label{app:demistify-ablation}
We detail the experimental setup used to conduct our ablation studies directed towards understanding the learning and forgetting dynamics of rare and complex examples respectively.

\paragraph{Rare Examples}
The experiments to show the rate of learning for rare examples are inspired by the singleton hypothesis as proposed by~\citet{Feldman2020DoesLR}. The hypothesis was inspired by a long-tailed distribution of visibility patterns in the person and bus category of the PASCAL dataset. For example, the dataset contains many buses with the front visible, but very few buses that were captured from the rear or the side, and even fewer buses whose view is occluded by the presence of other objects infront of them. (Refer to Figure 1 in their work for more details.) In our work, we first attempted to leverage the same training set-up with the provided visibility patters. However, we noted that there wasn't a strong correlation between the frequency of an example's visibility pattern, and the rate at which it was learnt. We hypothesize that this is because there are other factors of example hardness that may make an example be learnt slowly or fast (such as complexity, as detailed in the next paragraph). 
This can lead to an example being learnt fast if it has a simple pattern yet occurs rarely. Especially when there are only $O(1)$ samples from a given sub-group (based on the visibility pattern), we can not make any claims based on singleton correlation alone.

Hence, in order to distill the frequency of occurrence of an example with other confounders that may influence its training-time, we created a long-tailed dataset from the CIFAR-100 dataset. CIFAR-100 is a dataset of 100 object classes, which can be further grouped into 20 super-classes. For instance, examples from categories \emph{maple, oak, palm, pine, willow} all belong to the `superclass' of \emph{trees}. Similar division of 5 sub-classes is provided in the datasets for each of the superclasses. Each class contains 500 training examples and the overall dataset has 50,000 training examples.

As a first step towards creating a long-tailed dataset, we assign a fixed frequency ordering within the subgroups of a superclass. The most frequent subgroup has 500 examples in the training set, for the next most frequent subgroup, we randomly select 250 examples from the training set, and so on until the last sub group with 31 examples in the training set. This means that there are exactly 20 sub-groups in the final dataset with \{500,250,125,62,31\} examples respectively. Irrespective of the class number, the task is to predict the corresponding superclass, that is, we reduce the problem to a 20-class classification problem. However, we track the learning and forgetting dynamics of examples from each of the 100 sub-groups separately, based on their group frequency. To remove any other confounders of example hardness, we (i) randomize the group frequency ordering of the sub-groups within a superclass (in case some classes are harder to learn than the others); and (ii) randomize the examples that were selected based on the group size (in case some examples were ambiguous or hard). We further split the dataset into two IID partitions to analyze the learning time and SSFT, and average the results over 20 random runs of the experiment. Experimental results are detailed in the main paper.

\paragraph{Complex Examples}
Prior works advocating for, and understanding the simplicity bias~\citep{shah2020pitfalls} have operationalized the notion of simplicity via the complexity of hypothesis class required to learn the distribution that a complex example may be sampled from. In particular, \citet{shah2020pitfalls} construct a synthetic dataset with MNIST and CIFAR-10 images vertically stacked on top of each other---with the part with MNIST images corresponding to the part of the combined image with \emph{simpler} features, and the part with CIFAR-10 images corresponding to the part with \emph{complex} features.
They show that the model almost completely relies on the part of the image containing the MNIST digit even when it is less predictive of the true label. Inspired by this argument about the simplicity of features, we create a dataset that has the the union of images from the MNIST and the CIFAR-10 dataset. More specifically, we select classes from the MNIST dataset corresponding to digits $\{0,1,2,3\}$, and classes from the CIFAR-10 dataset corresponding to \{horses, airplanes, dog, frog\} and label them from $\{0,1,2,3\}$. This means that the model associates the label 0 to both the digit 0 and airplane class. The attempt of this experiment is to draw the link between the simplicity bias and the rate of learning. Experimental results are provided in the main paper.

\subsection{Stability of our metric}
\label{app:stability}
\begin{figure*}[t]
\centering
\begin{subfigure}[t]{0.47\linewidth}
  \includegraphics[width=\linewidth]{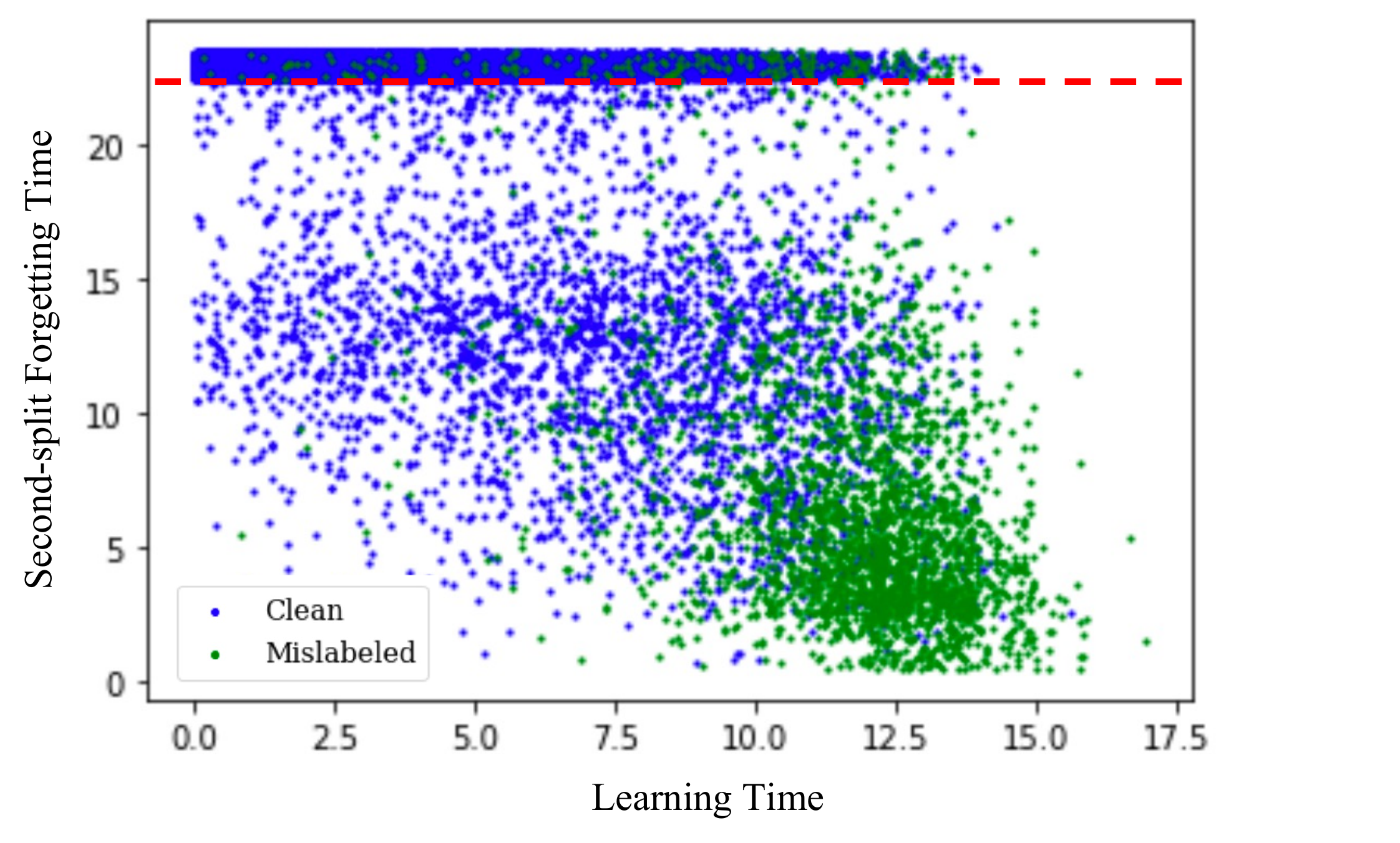}
  \caption{}
\end{subfigure}
\hspace{5mm}
\begin{subfigure}[t]{0.47\linewidth}
   \includegraphics[width=\linewidth]{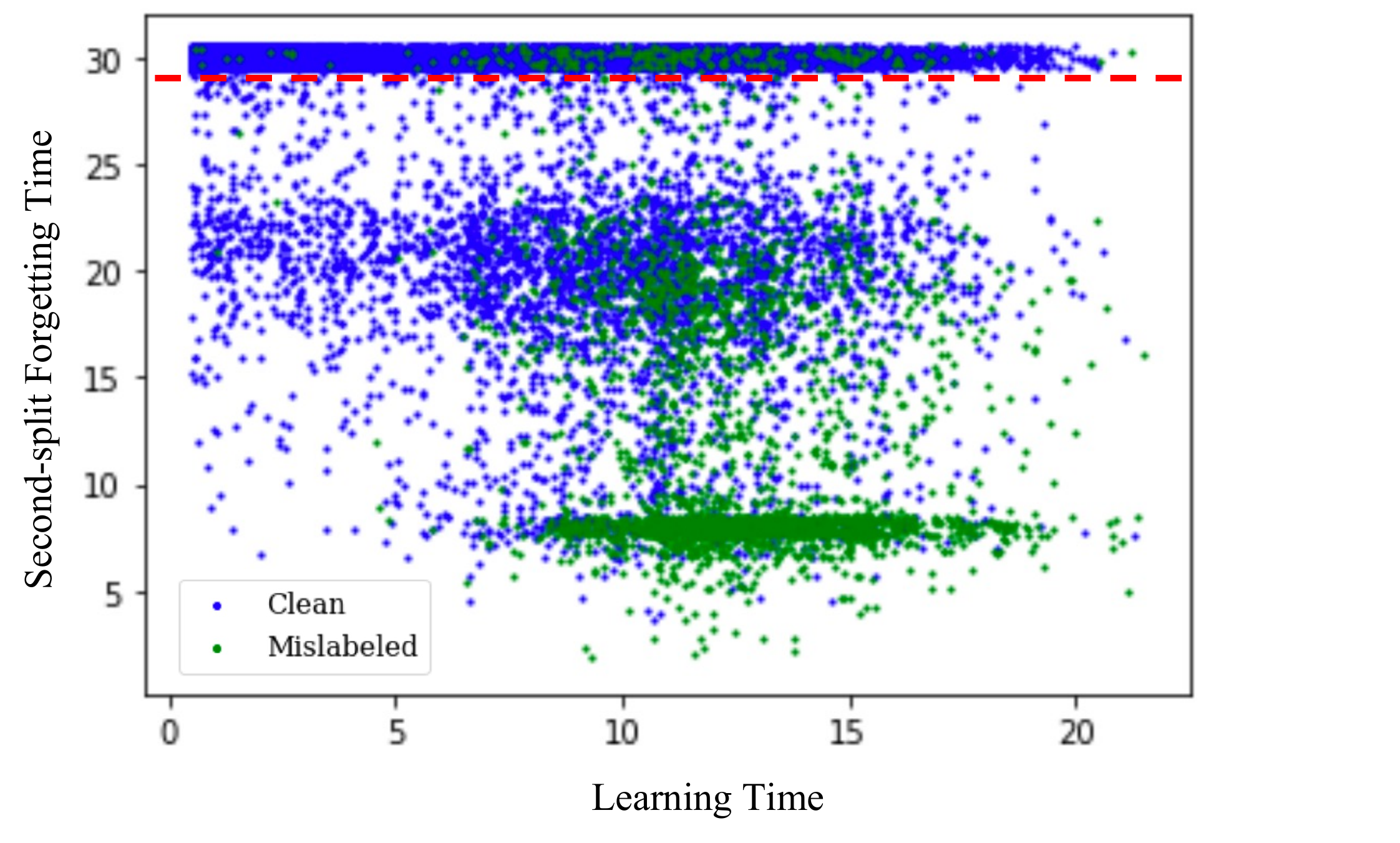}
   \caption{}
\end{subfigure}
\caption{FSLT (First-split learning time) is able to provide some degree of separation between mislabeled and clean samples when trained with the SGD optimizer (left), but fails when the model is trained using Adam (right) on the CIFAR-10 dataset. }
 \label{fig:optimizer_dependence}
\end{figure*}

\paragraph{Stability across architectures}
The forgetting of examples is a property of both the dataset and the model architecture. As a result, we find that just like the learning time, the forgetting time has a lower correlation between architectures. The average pearson correlation between the ResNet-9 and ResNet-50 models is 0.62 in case of the CIFAR10 dataset. However, we note that the most forgotten examples generalize across datasets. That is, the average pearson correlation between the bottom 10\% examples of the dataset is 0.87. This highlights how the forgetting metric is good for finding misaligned examples in the dataset, since they are not a property of the model architecture. We suspect that among the examples that are infrequently forgotten, the model campacity and other inductive biases of the model architecture may have a role in driving the average pearson correlation low. 
\paragraph{Stability across optimizers}
\citet{jiang2020characterizing} showed that changing the learning optimizer from SGD to Adam can lead to a significant change in the learning rate of examples from different levels of hardness (based on their regularity metric). More specifically, they find that examples with a low consistency score (closely correlated with learning speed) also get learnt fast when using the Adam optimizer. This suggests that using an optimizer like Adam at training time may have an impact on the ability of learning time based metrics to separate examples. In Figure~\ref{fig:optimizer_dependence}, we contrast the  ability of forgetting and learning time based metrics for identifying label noise when using the SGD and Adam optimizers. When using an optimizer such as SGD, the mislabeled samples are learnt slower than  a large fraction of the training examples, and the learning time metric offers some degree of separation between the clean and mislabeled examples. However, when we use the Adam optimizer, it results in joint learning of a large fraction of both mislabeled and clean samples. Hence, offering a very low degree of separation. However, under the same training procedure, the SSFT still allows us to distinguish between the mislabeled and clean samples.

\paragraph{Stability across seeds and learning rates}
The pearson correlation for stability across seeds for the forgetting time metric is 0.83. This is higher than the corresponding learning time based metric (correlation 0.56).
However, one of the drawbacks of our proposed metric is that the SSFT requires the use of an appropriate learning rate that allows the examples to be forgotten slowly. We provide more information about the same in the main paper.

\begin{figure*}[t]
\centering
\begin{subfigure}[t]{0.47\linewidth}
  \includegraphics[width=\linewidth]{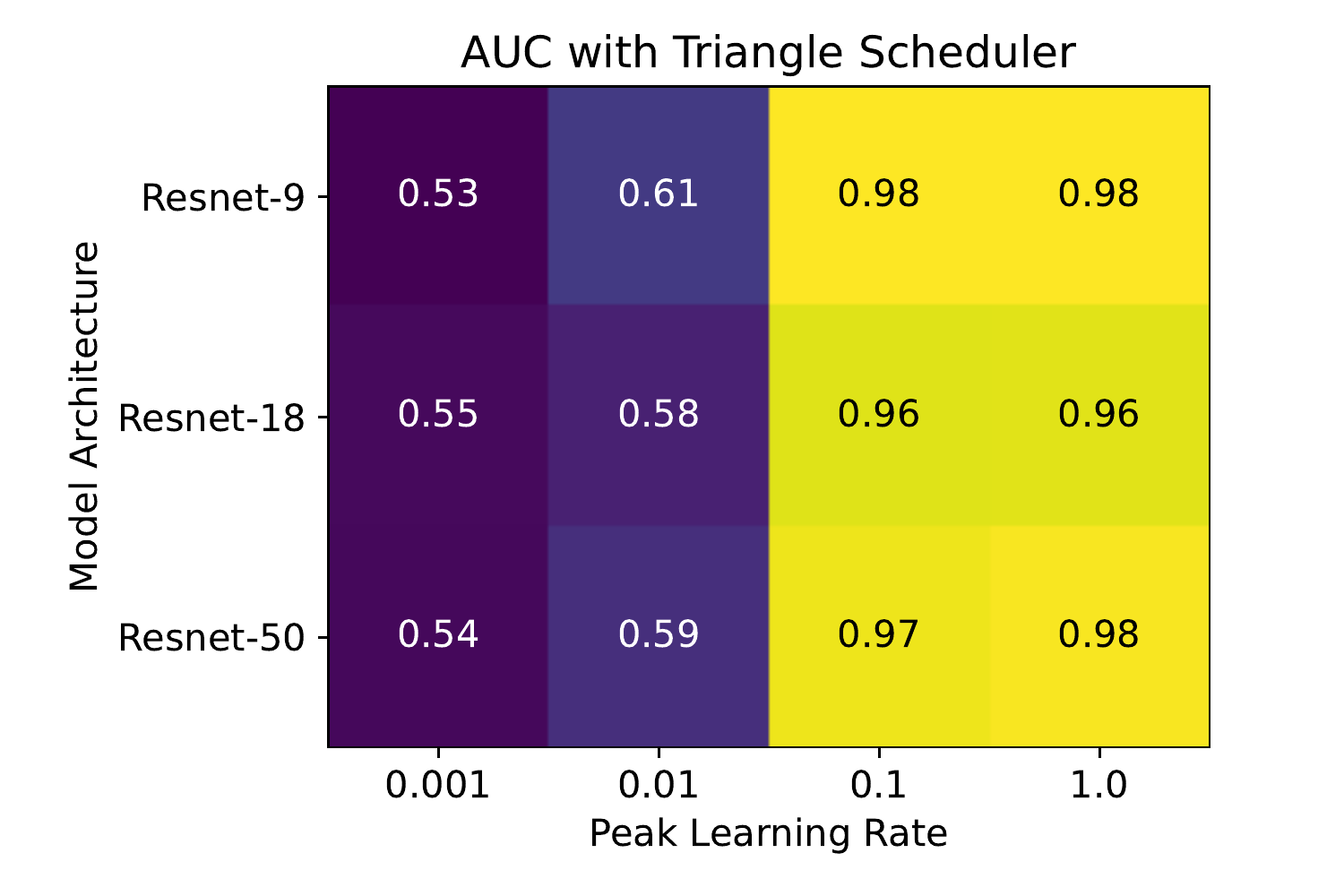}
  \caption{Triangular Learning Schedule}
\end{subfigure}
\hspace{2mm}
\begin{subfigure}[t]{0.47\linewidth}
   \includegraphics[width=\linewidth]{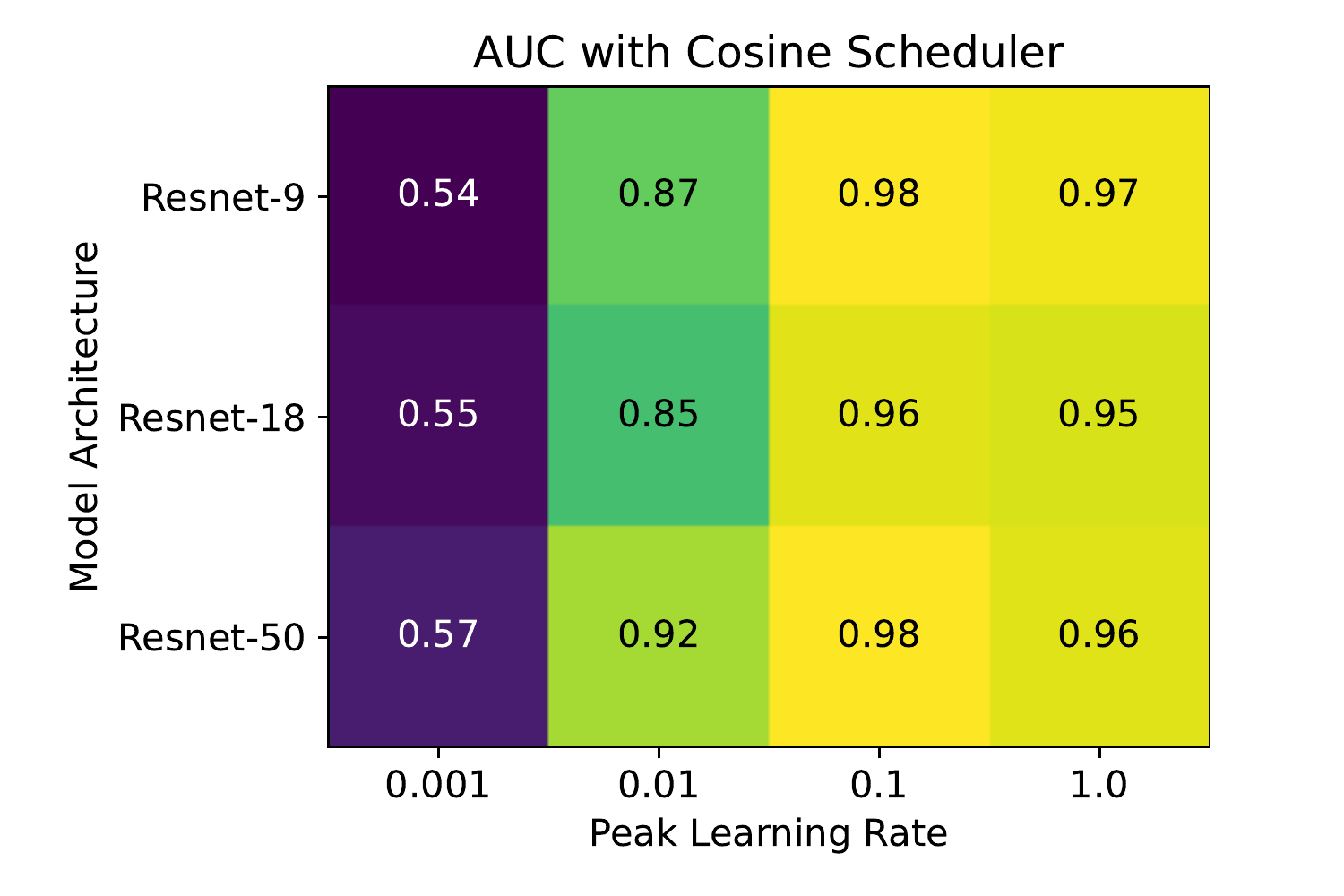}
   \caption{Cosine Learning Schedule}
\end{subfigure}
\hspace{2mm}
\begin{subfigure}[t]{0.47\linewidth}
   \includegraphics[width=\linewidth]{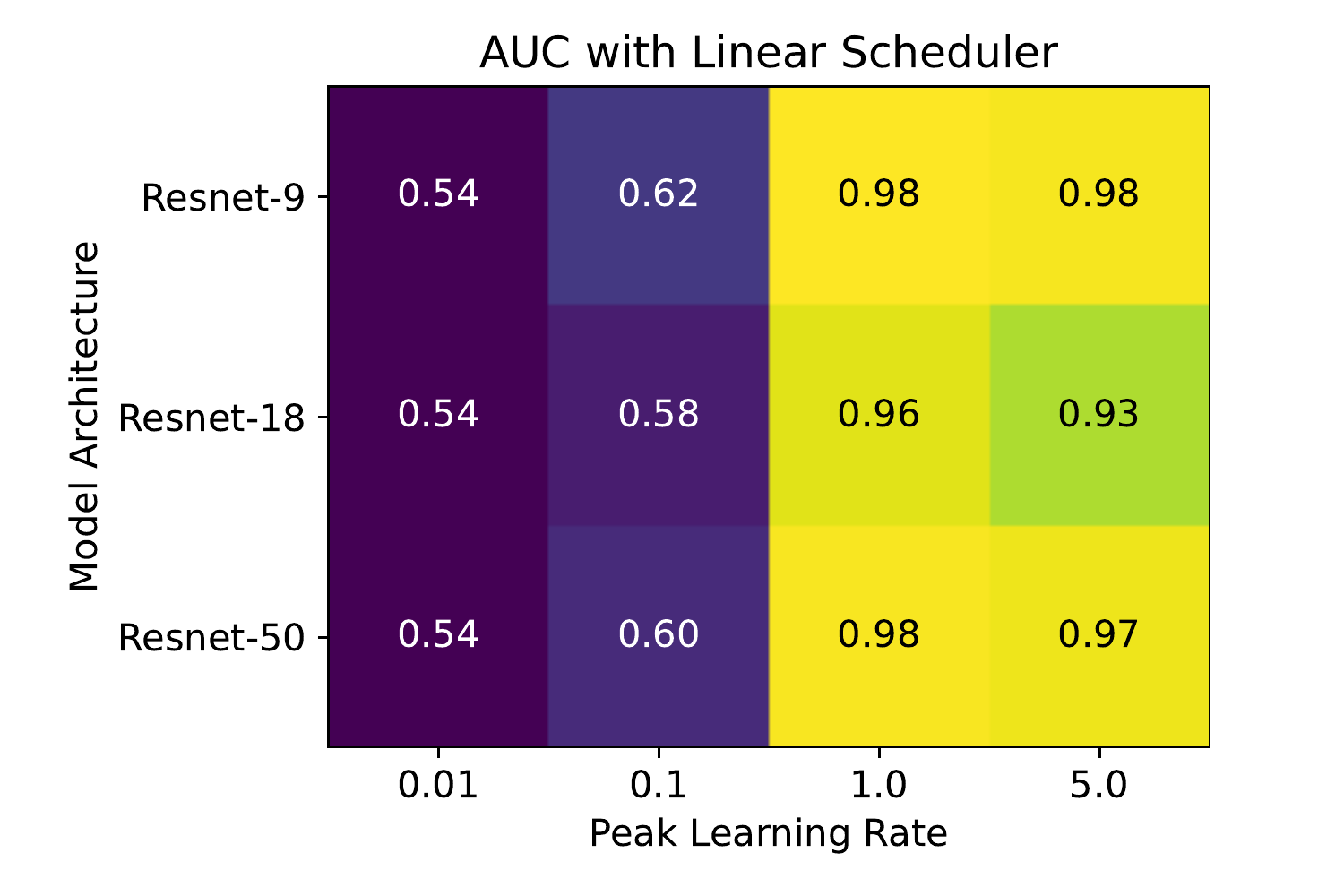}
   \caption{Linear Learning Schedule}
\end{subfigure}
\caption{We present a heatmap for the AUC of detection of mislabeled examples using the SSFT metric under various learning rates, architecture sizes and learning rate schedules for the SGD optimizer. The experiment was performed on the CIFAR-10 dataset with 10\% label noise, and the forgetting times were averaged over 5 seeds before using them for AUC calculation.}
 \label{fig:heatmap}
\end{figure*}

\paragraph{Stability Across Learning Rate Schedules}
We experiment with three different learning rate schedules for the second-split training---triangular, cosine and linear. In triangular learning rate, we increase the learning rate from 0 to the maximum set value linearly over the first 10 epochs, and then decay it back to 0 until we reach the last epoch (maximum of 100 epochs). In case of the linear schedule, we increase the learning rate from 0 to the maximum set value linearly over the course of 100 epochs. The intuition behind using a linear learning rate schedule was to be able to set higher peak learning rates so that the model eventually forget all the examples in the first split and we can create a better ordering between samples based on forgetting time (as opposed to the setting where only a small fraction of examples are ever forgotten).

The results of the combined analysis across hyperparameters such as architecture, learning rate and learning rate schedule are presented in the heatmap of AUC of mislabeled example detection in Figure~\ref{fig:heatmap}. 
The experiment was performed on the CIFAR-10 dataset with 10\% label noise, and the forgetting times were averaged over 5 seeds before using them for AUC calculation.
We can see that uniformly across architectures and learning rate schedules having a very low learning rate makes nearly all examples indistinguishable based on forgetting time. This is because all the models are sufficiently overparametrized to memorize all the examples in the dataset. Hence, when using a very small learning rate the optimization step moves the model weights insignificantly and we do not forget many mislabeled samples from the first split. On the other side, having a large learning rate helps achieve a strong separation between mislabeled and clean examples which also shows up in the form of high AUC values in the figure.

\begin{figure*}[t]
\centering

  \includegraphics[width=0.5\linewidth]{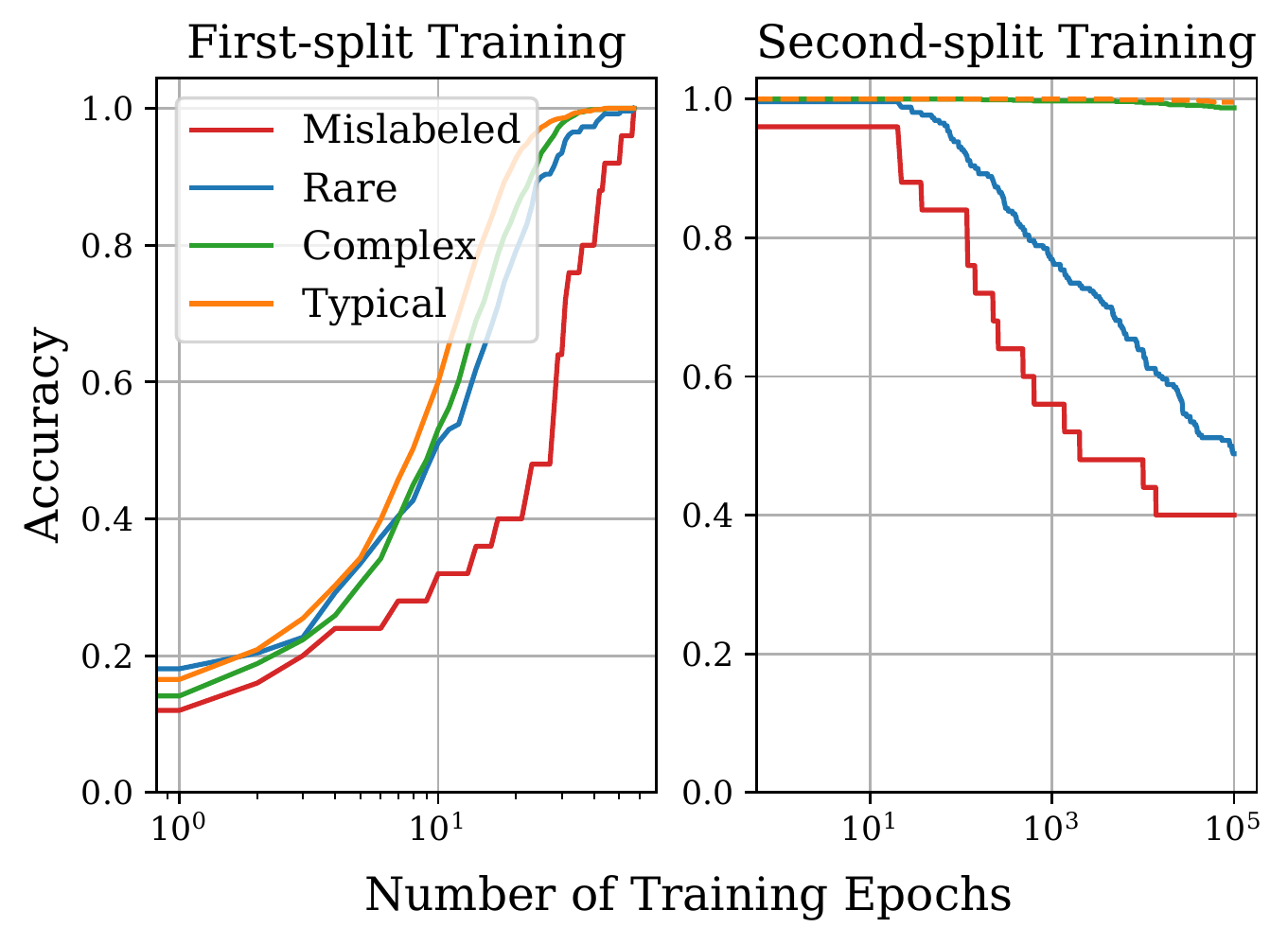}

\caption{Learning and forgetting curves for mislabeled and rare examples when the mislabeled examples are drawn from rare subgroups in the synthetic setup described in Section~\ref{subsec:spectrum}. }
 \label{fig:synthetic-rare-mislabeled}
\end{figure*}

\subsection{Impact of Sampling Frequency of Mislabeled Examples}
In the synthetic experiment performed in Section~\ref{subsec:spectrum}, we assumed that mislabeled examples occur from the majority subgroups. As a result, we observed that they get forgotten quickly during second-split training. However, in this section we aim to understand the impact of sampling frequency on the forgetting time of mislabeled examples. More specifically, we now assume that mislabeled samples occur in rare subgroups in the synthetic setup. We find that the learning curves of the mislabeled example stays the same as before, but the forgetting time for mislabeled examples closely approaches that of rare examples. This is because there is very little signal for the model to learn the opposite class during second split training since the example occurs only O(1) times. The learning and forgetting curves pertaining to the same experiment are presented in Figure~\ref{fig:synthetic-rare-mislabeled}. In contrast with the forgetting curve in Section~\ref{subsec:spectrum} where the mislabeled examples are quickly forgotten and their prediction is flipped, we find that when the subgroup corresponding to the mislabeled examples is infrequent, their forgetting time closely corresponds to that of rare examples; and on aggregate their  predictions do not get flipped in the epochs that the model was trained for.

\subsection{Mislabeled Example Detection}
In this section we provide additional details about the experimental setting for the results presented in Table~\ref{table:auc}. 
In case of the CIFAR-100 dataset, we reduce the learning rate for second-split training by a factor of 10, and use batch size of 128. While all the other training procedures in this paper used a cyclic learning rate, for the case of CIFAR-100, we use warm-up based multi-step decay learning rate schedule.\footnote{We follow the code in \href{https://github.com/weiaicunzai/pytorch-cifar100}{https://github.com/weiaicunzai/pytorch-cifar100}} The model used for training was ResNet-18, and 10\% label noise was added. The training setting for EMNIST is identical to that of the MNIST dataset. We use the first 10 classes of the dataset to make it a 10-class classification problem.

\end{document}